\documentclass[10pt]{article}
\pdfoutput=1
\usepackage{graphicx,graphics}
\usepackage[square,numbers]{natbib} 
\DeclareGraphicsExtensions{.pdf,.png,.jpg}
\usepackage{xspace,xcolor}
\usepackage{fullpage}

\usepackage{amsmath}
\usepackage{amsthm}
\usepackage{amsfonts}

\usepackage{algorithm}
\usepackage{algorithmicx}
\usepackage[noend]{algpseudocode}

\usepackage{amsmath,amsfonts}
\usepackage{amsthm}
\usepackage{amssymb,amsopn}
\usepackage{bm,bbm} % bold symbol

\usepackage{amssymb}
\usepackage{amsmath,amsfonts}
\usepackage{amsthm,amsopn}

\usepackage{bm} % bold symbol

\newcommand{\vct}[1]{\boldsymbol{#1}} % vector
\newcommand{\mat}[1]{\boldsymbol{#1}} % matrix
\newcommand{\field}[1]{\mathbb{#1}}
\newcommand{\R}{\field{R}} % real domain
\newcommand{\T}{^{\textrm T}} % transpose

\newcommand{\ProbOpr}[1]{\mathbb{#1}}
\newcommand{\expect}[2]{%
\ifthenelse{\equal{#2}{}}{\ProbOpr{E}_{#1}}
{\ifthenelse{\equal{#1}{}}{\ProbOpr{E}\left[#2\right]}{\ProbOpr{E}_{#1}\left[#2\right]}}} % Expectation: syntax: E{1}{2} = E_1[2], E{}{2}=E[2], E{1}{} = E_1
\newcommand{\var}[2]{%
\ifthenelse{\equal{#2}{}}{\ProbOpr{VAR}_{#1}}
{\ifthenelse{\equal{#1}{}}{\ProbOpr{VAR}\left[#2\right]}{\ProbOpr{VAR}_{#1}\left[#2\right]}}} % Expectation: syntax: V{1}{2} = V_1[2], V{}{2}=V[2], V{1}{} = V_1
\DeclareMathOperator{\argmax}{arg\,max}

\newtheorem{thm}{Theorem}
\newtheorem{theorem}{Theorem}

\newcommand{\vx}{{\vct{x}}}

\newcommand{\vz}{{\vct{z}}}

\newcommand{\vw}{\vct{w}}

\newcommand{\vP}{\vct{P}}

\newcommand{\vphi}{\vct{\phi}}

\newcommand{\vdelta}{\vct{\delta}}
\newcommand{\vomega}{\vct{\omega}}

\newcommand{\vDD}{\vct{\Delta}}
\newcommand{\eat}[1]{}

\newcommand{\ibm}{\textsf{ibm}\xspace}
\newcommand{\rbm}{\textsf{rbm}\xspace}
\newcommand{\onek}{\textsf{1-k}\xspace}
\newcommand{\ak}{\textsf{a-2-k}\xspace}
\newcommand{\mk}{\textsf{m-2-k}\xspace}
\newcommand{\ck}{\textsf{c-2-k}\xspace}
\newcommand{\ter}{\textsf{TER}\xspace}

\begin{document}

\title{How to Scale Up Kernel Methods\\
to Be As Good As Deep Neural Nets}

\author{Zhiyun Lu$^{1\dagger}$, Avner May$^{2\dagger}$\\
Kuan Liu$^{1\ddagger}$, Alireza Bagheri Garakani$^{1\ddagger}$, Dong Guo$^{1\ddagger}$, Aur\'{e}lien Bellet$^{4\ddagger}$\footnote{Most of the work in this paper was carried out
while the author was affiliated with Department of Computer Science,
University of Southern California.}\\
Linxi Fan$^2$, Michael Collins$^2$, Brian Kingsbury$^3$, Michael Picheny$^3$, Fei Sha$^{1\P}$  \\[0.5em]
$^1$ Dept. of Computer Science,  U. of Southern California, Los Angeles, CA 90089\\
  \texttt{\{zhiyunlu, kuanl, bagherig, dongguo, feisha\}@usc.edu}\\[0.5em]
$^2$ Dept. of Computer Science, Columbia University, New York, New York 10027\\
 \texttt{\{avnermay, mcollins\}@cs.columbia.edu}, \texttt{lf2422@columbia.edu}\\[0.5em]
$^3$ IBM T. J. Watson Research Center, Yorktown Heights, NY 10598  \\
 \texttt{\{bedk, picheny\}@us.ibm.com}\\[1em]
 $^4$ LTCI UMR 5141, T\'{e}l\'{e}com ParisTech \& CNRS, France\\
 \texttt{aurelien.bellet@telecom-paristech.fr}\\[1em]
 $^\dagger$ and $^\ddagger$: shared first and  second co-authorships, respectively\\[0.5em]
 $^\P$: to whom questions and comments should be sent
    }

\maketitle

\begin{abstract}
The computational complexity of kernel methods has often been a major barrier for applying them to large-scale learning problems. We argue that this barrier can be effectively overcome. In particular, we develop methods to scale up kernel models to successfully tackle large-scale learning problems that are so far only approachable by deep learning architectures. Based on the seminal work by \citep{rahimi07random} on approximating kernel functions with features derived from random projections, we advance the state-of-the-art by proposing methods that can efficiently train models with hundreds of millions of parameters,  and learn optimal representations from multiple kernels.  We conduct extensive empirical studies on problems from image recognition and automatic speech recognition, and show that the performance of our kernel models matches that of well-engineered deep neural nets (DNNs). To the best of our knowledge, this is the first time that a direct  comparison between these two methods on large-scale problems is reported. Our kernel methods have several appealing properties: training with convex optimization, cost for training a single model comparable  to DNNs, and significantly reduced total cost due to fewer hyperparameters to tune for model selection. Our contrastive study between these two very different but equally competitive models  sheds light on fundamental questions such as how to learn good representations.
\end{abstract}

% !TEX root = main.tex
\section{Introduction}

Deep neural networks (DNNs) and other types of deep learning architecture have made significant advances~\citep{bengio09deep,bengio13representation}. In both well-benchmarked tasks and real-world applications, such as automatic speech recognition~\citep{hinton12deep,mohamed12dbn,seide11dnn} and image recognition~\citep{krizhevsky12imagenet,szegedy14googlelenet}, deep learning architectures have achieved an unprecedented level of success and have generated major impact.

Arguably, the most instrumental factors contributing to their success are: (1) learning from a huge amount of training data for highly complex models with millions to billions of parameters; (2) adopting simple but effective optimization methods such as stochastic gradient descent; (3) combatting overfitting with new schemes such as drop-out~\citep{hinton12dropout}; and (4) computing with massive parallelism on GPUs. New techniques as well as ``tricks of the trade'' are frequently invented and added to the toolboxes for machine learning researchers and practitioners.

In stark contrast, there have been many fewer publicly known successful applications of kernel methods (such as support vector machines) to problems at a scale comparable to the speech and image recognition problems tackled by DNNs.  This is a surprising chasm, noting that kernel methods have been extensively studied both theoretically and empirically for their power of modeling highly nonlinear data~\citep{scholkopf02}. Moreover, the connection between kernel methods and (infinite) neural networks has also been long noted~\citep{neal94priors,williams96infinite,cho09arccos}.

Nonetheless, a common misconception is that it may be difficult, if not impossible, for kernel methods to catch up with deep learning methods in addressing large-scale learning problems.  In particular, many kernel-based algorithms scale quadratically in the number of training samples. This barrier in computational complexity makes it especially challenging for kernel methods to reap the benefits of learning from a very large amount of data, while deep learning architectures are especially adept at it.

We contend that this skepticism can be sufficiently attenuated. Concretely, in this paper, we investigate and propose new ideas tailored for kernel methods, with the aim of scaling them up to take on challenging problems in computer vision and automatic speech recognition. 

To this end, we build on the work by \citep{rahimi07random} on approximating kernel functions with features derived from random projections. Our innovation is, however, to advance the state-of-the-art to a much larger scale. Concretely, we propose fast training methods for models with hundreds of millions of parameters --- these methods are necessary for classifiers using hundreds of thousands of features to recognize thousands of categories. We also propose scalable methods for combining multiple kernels as ways of learning feature representations. Interestingly, we show multiplicative combination of kernels scale better than additive combinations. 

We validate our approaches with extensive empirical studies. We contrast kernel models to DNNs on 4 large-scale benchmark datasets, some of which are often used to demonstrate the effectiveness of DNNs. We show that the performance of large-scale kernel models approaches or surpasses their deep learning counterparts, which are either exhaustively optimized by us or are well-accepted as yardsticks in industry standards. 

While providing a recipe to obtain state-of-the-art large-scale kernel models, another important contribution of our work is to shed light on new perspectives and opportunities for future study. The techniques we have developed are easy to implement, readily reproducible, and incur much less computational cost (for hyperparameter tuning and model selection) than deep learning architectures. Thus, they are valuable tools, tested and verified to be effective for constructing comparative systems. 

Comparative studies enabled by such systems will, in our view be indispensable in pursuing the higher goal of exploring and acquiring an understanding of how the two camps of methods differ, for instance in learning new representations of the original data\footnote{Note this inquiry would be informative only if both kernel methods and deep learning methods attain similar performance yet exploit different aspects of data.}.  As an example, we show that combining kernel models and DNNs improves over either individual model, suggesting that the two paradigms learn different yet complementary representations from the data. We believe that research in this line will offer deep insights, and broaden the theory and practice of designing alternative methods to both DNNs and kernel methods for large-scale learning.

The rest of the paper is organized as follows. We  briefly review related work in section~\ref{sRelated}. In section~\ref{sRandom}, we give a brief account of \citep{rahimi07random}.  We describe our approaches in section~\ref{sApproach}. In section~\ref{sResults}, we
report extensive experiments comparing DNNs and kernel
methods on the problems in image and  automatic speech
recognition. We conclude and discuss future directions in section~\ref{sDiscuss}.

% !TEX root = main.tex
\section{Related work}
\label{sRelated}

The computational complexity of kernel methods, such as support vector machines, depends quadratically on the number of training examples at training time and linearly on the number of training examples at the testing time. Hence, scaling up kernel methods has been a long-standing and actively studied problem. \citep{largescalekernelmachines07} summarizes several earlier efforts in this vein. 

With clever implementation tricks such as computation caching  (for example, keeping only a small portion of the very large kernel matrix inside the memory), earlier kernel machines can cope with hundreds of thousands of samples~\citep{smolaXX,decoste02invariant}. \citep{bottou06svm} provides an excellent account of various design considerations.

To further reduce the dependency on the number of training samples, a more effective strategy is to actively select training samples~\citep{bottouXX}. An early version of this idea was reflected in the Sequential Minimal Optimization (SMO) algorithm~\citep{platt98smo}. With more sophistication, this technique was extended to enable training SVMs on 8 million samples~\citep{loosli06selective}. Alternative approaches exploit the equivalence between SVMs and sparse greedy approximation and solve SVMs approximately with a smaller subset of examples called coresets\footnote{We also experimented those techniques though we were not able to identify significant empirical success.}~\citep{tsang05cvm,clarkson10coreset}.  Exploiting structures of the kernel matrix  can scale kernel methods to 2 million to 50 million training samples~\citep{sonnenburg10coffin}.  Note that at the time of publication, none of the above-mentioned methods had been directly compared to DNNs.

Instead of reducing the number of training samples, we can reduce the dimensionality of kernel features. In theory, those features are infinite dimensional. But for any practical problem, the dimensionality is bounded above by the number of training samples. The main idea is then to directly use such features, after dimensionality reduction, to construct classifiers (i.e., solving the optimization problem of SVM in the primal space).  

Thus far, approximating kernels with finite-dimensional features has been recognized as a promising way of scaling up kernel methods. The most relevant one to our paper is the early observation by  \citeauthor{rahimi07random} that inner products between features derived from random projections can be used to approximate translation-invariant kernels, a direct result of spectral analysis of positive functions~\citep{berg84,scholkopf02,rahimi07random}. Their follow-up work of using those random features ---  weighted random kitchen sink~\citep{rahimi08kitchen} ---  is a major inspiration to our work.   

Since, there has  been  a growing interest in using random  projections to approximate different kernels~\citep{kar12random,hamid14compact,le13fastfood,vedaldi12additive}. For example, \citep{dai14scalable} studied how to use random features for online learning. We note that the amount of time for such classifiers  to make a prediction depends linearly on the number of training samples. This could be a concern when the number of training samples is large.

In spite of these progresses,  there have been only a few reported large-scale empirical studies of those techniques on challenging tasks from speech recognition and computer vision, on which DNNs have been  highly effective. In the context of automatic speech recognition, examples of directly using kernel methods  were reported~\citep{deng12convex,cheng11arcos,huang14kernel}. However, the tasks were fairly small-scale (for instance, on the TIMIT dataset).  Moreover, none of them explores kernel learning as a way of learning new representations.  In contrast, one major aspect of our work is to use multiple kernel learning to arrive at new representations so as to reduce the gap between DNNs and kernel methods, cf. section~\ref{sMKL}.

% !TEX root = main.tex
\section{Features from random projections}
\label{sRandom}
In what follows, we describe the basic idea we have built upon to scale up kernel methods. The technique is based on explicitly and efficiently constructing features  --- they are generated randomly ---  whose inner products then approximate kernel functions. Once such features are constructed, they can be used as inputs by any classifier.

\subsection{Generate features by random projections}

Given a pair of data points $\vx$ and $\vz$, a positive definite kernel function $k(\cdot, \cdot): \R^d \times \R^d \rightarrow \R$ defines an inner product between the images of the two data points under a (nonlinear) mapping $\vphi(\cdot): \R^d \rightarrow \R^M$,
\begin{equation}
k(\vx, \vz) = \vphi(\vx)\T \vphi(\vz)
\label{eKernelDef}
\end{equation}
where the dimensionality $M$ of the resulting mapping $\vphi(\vx)$ can be infinite (in theory).  

Kernel methods avoid inference in $\R^M$. Instead, they rely on the kernel matrix over the training samples. When $M$ is far greater than $N$, the number of training samples, this trick provides a nice computational advantage. However, when $N$ is exceedingly large, this complexity at the quadratic order of $N$  becomes impractical.

\citeauthor{rahimi07random} leverage a classical result in harmonic analysis and provide a fast way to approximate $k(\cdot, \cdot)$ with \emph{finite}-dimensional features~\citep{rahimi07random}:
\begin{theorem}{(Bochner's theorem, adapted from \citep{rahimi07random})}
A continuous kernel $k(\vx, \vz) = k (\vx - \vz)$ is positive definite if and only if $k(\vdelta)$ is the Fourier transform of a non-negative measure.
\end{theorem}
More specifically, for shift-invariant kernels such as Gaussian RBF and Laplacian kernels,
\begin{equation}
k^{\textsf{rbf}} = e^{-\|\vx-\vz\|_2^2/2\sigma^2},\quad k^{\textsf{lap}} = e^{-\|\vx-\vz\|_1/\sigma}
\label{eGauLap}
\end{equation}
 the theorem implies that the kernel function can be expanded with harmonic basis, namely
\begin{equation}
k(\vx - \vz)  = \int_{R^d} p(\vomega) e^{j\vomega\T(\vx-\vz)}\, d\vomega = \expect{\vomega}{e^{j\vomega\T\vx}e^{-j\vomega\T\vz}}
\label{eFourier}
\end{equation}
where $p(\vomega)$ is the density of a $d$-dimensional probability distribution. The expectation is computed on complex-valued functions of $\vx$ and $\vz$. For real-valued kernel functions, however, they can be simplified to the cosine and sine functions, see below. 

For Gaussian RBF and Laplacian kernels, the corresponding densities are Gaussian and Cauchy distributions:
\begin{equation}
\label{RBF}
p^{\textsf{rbf}}(\vomega) = N\left(0, \frac{1}{\sigma}\mat{I}\right),\quad p^{\textsf{lap}}(\vomega) = \prod_d \frac{1}{\pi(1+\sigma^2\omega_d^2)}
\end{equation}

The harmonic decomposition suggests a sampling-based approach of  approximating the kernel function. Concretely, we draw $\{\vomega_1, \vomega_2, \ldots, \vomega_D\}$ from the distribution $p(\vomega)$ and use the sample mean to approximate
\begin{equation}
k(\vx, \vz) \approx \frac{1}{D} \sum_{i=1}^{D} \phi_{\vomega_i}(\vx) \phi_{\vomega_i}(\vz) =   \hat{\vphi}(\vx)\T\hat{\vphi}(\vz)
\label{eApprox}
\end{equation}
The \emph{\textbf{random feature vector}} $\hat{\vphi}$ is thus composed of scaled cosines of random projections
\begin{equation}
\phi_{\vomega_i}(\vx) =  \sqrt{\frac{2}{D}} \cos (\vomega_i\T\vx + b_i)
\label{eRandomFeature}
\end{equation}
where $b_i$ is a random variable, uniformly sampled from $[0,\ 2\pi]$. Details on the convergence property of this approximation can be found in~\citep{rahimi07random}.

A key advantage of using approximate features  over standard kernel methods is its
scalability to large datasets. Learning with
a representation $\hat{\vphi}(\cdot) \in \R^D$ is relatively efficient
provided that $D$ is far less than the number of training samples.  For example, in our experiments (cf. section~\ref{sResults}), we have $7$ million to $10$ million  training samples, while $D=50,000$ often leads to good performance.

\subsection{Use random features in classifiers}
\label{sRKS}

Just as the standard kernel methods (SVMs or kernelized linear regression) can be seen as fitting data with linear models in kernel-induced feature spaces, 
we can plug in the random feature vector $\hat{\vphi}(\vx)$ in just about any (linear) model. In this paper, we focus on using them to construct multinomial logistic regression. Specifically, our model is a special instance of the \emph{weighted sum of random kitchen sinks}~\citep{rahimi08kitchen}
\begin{equation}
p( y = c |\vx) = \frac{e^{\vw_c\T \hat{\vphi}(\vx)}}{\sum_{c}e^{\vw_c\T \hat{\vphi}(\vx)}}
\label{eMLR}
\end{equation}
where the label $y$ can take any value from $\{1, 2, \ldots, C\}$. We use multinomial logistic regression mainly because it can deal with a  large number of classes  and provide posterior probability assignments, needed by the application task (i.e., the speech recognition systems, in order to combine with components such as language models).

% !TEX root = main.tex
\section{Our Approaches}
\label{sApproach}

To scale up kernel methods, we address two challenges:  (1) how to train large-scale models in the form of eq.~(\ref{eMLR}); (2) how to learn optimal kernel functions adapted to the data. We tackle the former with a parallel optimization algorithm and the latter by extending the construction of random features initially proposed in ~\citep{rahimi07random}.

\subsection{Parallel optimization for large-scale kernel models}
\label{sParallel}

While random features and weighted sum of random kitchen sinks have been investigated before, there are few reported cases of  scaling up to the problems commonly seen in automatic speech recognition and other domains. For example, in our empirical studies of acoustic modeling (cf. section~\ref{sASR}), the number of classes is $C=1000$ and we often use more than $D = 100,000$ random features to compose $\hat{\vphi}(\vx)$.  Thus, the linear model eq.~(\ref{eMLR}) has a large number of parameters (about $C \times D = 10^8$). 

We have developed two major strategies to overcome this challenge. First, we leverage the observation that fitting multinomial logistic regression is a convex optimization problem and adopt the method of stochastic averaged gradient (SAG) for its faster convergence, both theoretically and empirically, over stochastic gradient descent (SGD)~\citep{roux12sag}. Note that while SGD is widely applicable to both convex and non-convex optimization problems,  SAG is specifically designed for convex optimization and thus well-suited to our learning setting.

Secondly, we leverage the property that random projections are just \emph{random} -- that is, given a $D$-dimensional $\hat{\vphi}(\vx)$, any random subset of it is still random. Our idea is then to train a model on each subset of features \emph{in parallel} and then assemble them together to form a large model.

Specifically, for large $D$ (say $\ge 100,000$), we partition $D$ into $B$ blocks $\hat{\vphi}_b(\vx)$ with each block having a size of $D_0$ (say $ = 25,000$). Note that each block corresponds to a different set of random projections sampled from the density $p(\vomega)$. 
We train $B$ multinomial logistic regression models and obtain $B$ sets of parameters for each class, ie., $\{\vw_c^{b}, c=1, 2,\ldots, C\}_{b=1}^B$. To assemble them, we combine in the spirit of \emph{geometric mean} of the probabilities (or arithmetic mean of the $\log$ probabilities)
\begin{equation}
\begin{aligned}
p( y = c|\vx)  & \propto \exp\left(\frac{1}{B}\sum_{b=1}^B \hat{\vphi}_b(\vx)\T \vw_c^b\right)\\
 \\
 & = \sqrt[^B]{\prod_b \exp\left(\hat{\vphi}_b(\vx)_b\T \vw_c^b \right)}
 \end{aligned}
\label{eAssembled}
\end{equation}
Note that this assembled model can be seen as a $D$-dimensional model with parameters of $\{\frac{1}{B}\vw_c^{b}, c=1, 2,\ldots, C\}_{b=1}^B$.

We sketch the main argument for the validity of this parallel training procedure, leaving a rigorous proof for future work. The parameters of the weighted sum of random kitchen sink converges in $O(1/\sqrt{D})$ to the true risk minimizer~\citep{rahimi08kitchen}. Thus, for each model of size $D_0$, the pre-softmax activations (i.e., the logits) converge in $O(1/\sqrt{D_0})$. For $B$ such models, the arithmetic mean of logits  converge in $O(1/(\sqrt{B}\sqrt{D_0}))$ thus matching up the rate for a $D$-dimensional model.  Our extensive empirical studies have supported our argument --- in virtually all training settings, the assembled models cannot be improved further, attaining the optimum of the corresponding $D$-dimensional model.

% !TEX root = main.tex
\subsection{Learning kernel features}
\label{sMKL}

Another advantage of using kernels is to sidestep the problem of feature engineering, i.e., how to select the best basis functions for a task at hand. Essentially, determining what kernel function to use implicitly specifies the basis functions. But then the question becomes: how to select the best kernel function?

One popular paradigm to address the latter problem is multiple kernel learning (MKL)~\citep{lanckriet04mkl,bach04mkl,cortes09nonlinear,kloft11lp}. That is, starting from a collection of base kernels, the algorithm identifies the best subset of them and combines them together to best adapt to the training data, analogous to designing the best features according to the data.

In the following,  we show how a few common  MKL ideas can benefit from the previously described large-scale learning techniques (cf. section~\ref{sRandom}). While many MKL algorithms are formulated with kernel matrices (and thus are not easily scalable to large problems), we demonstrate how they can be efficiently implemented with the general recipe of random feature approximation.  
Among them, we show an interesting and novel result on combining kernels with Hadamard products, where the random feature approximation is especially computationally advantageous. 

In our empirical studies (detailed in Supplementary Material), we will show that MKL improves methods using a single kernel, and eventually approaches the performance of deep neural networks.  Thus, MKL presents an effective and computational tractable alternatives to DNNs, even for large-scale problems.

\paragraph{Additive Kernel Combination}

Given a collection of base kernels $\{k_i(\cdot, \cdot), i=1, 2, \ldots, L\}$, their non-negative combination
\begin{equation}
k(\cdot, \cdot) = \sum_i \alpha_i k_i(\cdot, \cdot)
\end{equation}
is also a kernel function, provided $\alpha_i \ge 0$ for any $i$. 

Suppose each kernel $k_i(\cdot, \dot)$ is approximated with a $D$-dimensional random feature vector $\hat{\vphi}_i(\cdot)$, as in eq.~(\ref{eApprox}). Then, given the linearity of the combination, the kernel function $k(\cdot, \cdot)$ can be approximated by
\begin{equation}
\label{eAKC}
k(\cdot, \cdot) \approx \sum_i \alpha_i \hat{\vphi}_i(\cdot)\T\hat{\vphi}_i(\cdot) = \hat{\vphi}(\cdot)\T\hat{\vphi}(\cdot)
\end{equation}
where $\hat{\vphi}(\cdot)$ is just the concatenation of the $\sqrt{\alpha_i}$-scaled $\hat{\vphi}_i(\cdot)$. Note that the dimensionality of $\hat{\vphi}(\cdot)$ would be $L \times D$.

There are several ways to exploit this approximation. The first way is to straightforwardly plug $\hat{\vphi}(\cdot)$ into the multinomial logistic regression eq.~(\ref{eMLR}) and optimize over $L\times D$ features. The second way is more scalable. For each $\hat{\vphi}_i(\cdot)$, we learn an optimal model with parameters $\vw_c^i$ for each class $c$. We then learn a set of combination coefficients $\alpha_i$ by optimizing the likelihood model
\begin{equation}
p( y = c |\vx) \propto \exp\left(\sum_i \sqrt{\alpha_i}  \hat{\vphi}_i(\vx)\T\vw_c^i\right)
\end{equation}
while holding the other parameters fixed. This is a convex optimization with (presumably) a small set of parameters. 

While the first approach is more general, however, empirically, we do not observe a strong difference and have adopted the second approach for its  scalability.

\paragraph{Multiplicative Kernel Combination}  

Kernels can also be multiplicatively combined from base kernels:
\begin{equation}
k(\cdot, \cdot) = \prod_{i=1}^L k_i (\cdot, \cdot)
\end{equation}
Note that this is a highly nonlinear combination~\citep{cortes09nonlinear}. Unlike the additive combination, to approximate the multiplicative combination of kernels, there does not exist a simple form (such as concatenating) of composing with the approximate features of individual kernels. Nonetheless,
We have proved the following theorem as a way to constructing the approximate features for $k(\cdot, \cdot)$ efficiently.
\begin{theorem}
\label{thMulti}
Suppose all $k_i(\cdot, \cdot)$ are translation-invariant kernels such that
\begin{equation}
k_i(\vx - \vz)  = \int_{R^d} p_i(\vomega) e^{j\vomega\T(\vx-\vz)}\, d\vomega 
\end{equation}
Then $k(\cdot, \cdot)$ is also translation-invariant such that
\begin{equation}
k(\vx - \vz)  = \int_{R^d} p(\vomega) e^{j\vomega\T(\vx-\vz)}\, d\vomega 
\end{equation}
where the probability measure $p(\vomega)$ is given by the convolution of all $p_i(\vomega)$
\begin{equation}
p(\vomega) = p_1(\vomega) \ast p_2(\vomega)  \ast \cdots \ast p_L(\vomega)
\end{equation}
Moreover, let $\vomega_i \sim p_i(\vomega)$ be a random variable drawn from the corresponding distribution, then 
\begin{equation}
\vomega = \sum_i \vomega_i \sim p(\vomega)
\end{equation}
Namely, to approximate $k(\cdot, \cdot)$, one needs only to draw random variables from each individual component kernel's corresponding density, and use the sum of those variables to compute random features.
\end{theorem}
The proof of the theorem is in the Suppl.. We note that $\vomega$ and $\vomega_i$ have the same dimensionality. Thus, the number of approximating features is independent of the number of kernels, leading to a computational advantage over additive combination.

\paragraph{Kernel composition} 

Kernels can also be composited. Specifically, if $k_2(\vx, \vz)$ is a kernel function that depends on only the inner products of its arguments, then $k = k_2 \circ k_1$ is also a kernel function. A concrete example is when $k_2$ is the Gaussian RBF kernel  and $k_1(\vx, \vz) = \vphi_1(\vx)\T\vphi_1(\vz)$ for some mapping $\vphi_1(\cdot)$
\begin{equation*}
\begin{aligned}
k(\vx, \vz)  & 
= \exp\{-\| \vphi_1(\vx) - \vphi_1(\vz)\|_2^2/\sigma^2\}\\
& = \exp\{ - [k_1(\vx, \vx)+ k_1(\vz, \vz)- 2 k_1(\vx, \vz)]/\sigma^2\}
\end{aligned}
\end{equation*}

If we approximate both $k_1$ and $k_2$ using the random feature approximation of eq.~(\ref{eApprox}), the composition would be (graphically) equivalent to the following mapping,
\begin{equation}
\vx \xrightarrow{\vomega \sim p_1(\vomega)} \hat{\vphi}_1(\vx) \xrightarrow{\vomega \sim p_2(\vomega)} \hat{\vphi}_2(\hat{\vphi}_1(\vx))
\end{equation}
namely, a one-hidden-layer  neural networks with the weight parameters in the layers being completely random. As before, the result of the composite mapping $\hat{\vphi}_2 \circ \hat{\vphi}_1$ can be used in any classifier as  input features. 

We also generalize this operation to introduce a linear projection to reduce dimensionality, serving as information bottleneck:  $\hat{\vphi}_2 \circ \vP \circ \hat{\vphi}_1$.  We experimented on two choices.

First, $\vP$ performs PCA (using the sample covariance matrix) on $\hat{\vphi}_1(\vx)$. Note that this implies $\vP \circ \hat{\vphi}_1$ is an approximate kernel PCA on the original feature space $\vx$, using the kernel $k_1$.
Secondly, $\vP$ performs \emph{supervised} dimensionality reduction. One simple choice is to implement Fisher discriminant analysis (FDA) on $\hat{\vphi}_1(\vx)$, which is equivalent to kernel (FDA) on $\vx$. In our experiments, we have used a different procedure  in a similar spirit. Specifically, 
In particular, we first use $\hat{\vphi}_1(\vx)$ as input features to build a multinomial logistic regression to predict its labels. We then perform PCA on the $\log$-posterior probabilities.   Our choice here is largely due to the consideration of re-using the computations as we often need to estimate the performance of $k_1(\cdot, \cdot)$ alone, thus the multinomial classifier built with $k_1(\cdot, \cdot)$ is readily usable.

% !TEX root = main.tex
\section{EXPERIMENTAL RESULTS}
\label{sResults}

We validate our approaches of scaling up kernel methods on challenging problems in computer vision and automatic speech recognition (ASR). We conduct extensive empirical studies comparing kernel methods to deep neural networks (DNNs), which perform well in computer vision, and are state-of-the-art in ASR. We show that kernel methods attain similarly competitive performance as DNNs -- details in section~\ref{sImage} and ~\ref{sASR} (as well as Supplementary Material).

What can we learn from two very different, yet equally competitive, learning models? We report our initial findings on this question (section~\ref{sComparison}). We show that kernel methods and DNNs learn different yet complementary representation of the data. As such, a direct application of this observation is to combine them to obtain better models than either independently. 

\subsection{General setup}

For all kernel-based models, we tune only three hyperparameters: the bandwidths for Gaussian or Laplacian kernels, the number of random projections, and the step size of the (convex) optimization procedure (as adjusting it has a similar effect as early-stopping).

For all DNNs, we tune hyperparameters related to both the architecture and the optimization. This includes the number of layers, the number of hidden units in each layer, the learning rate, the rate decay, the momentum, regularization, etc. We also use unsupervised pre-training and tune hyperparameters for that phase too.

Details about tuning those hyperparameters are described in the Supplementary Material as they are often dataset-specific.   In short,  model selection for kernel models has significantly lower computational cost. We give concrete measures to support this observation in section~\ref{sTiming}.

% !TEX root = main.tex
\begin{table}
\small
\centering
\caption{Handwritten digit recognition error rates (\%)}
\label{tMNISTSummary}
\vspace{0.5em}
\begin{tabular}{lcccc }
\hline
Model     & \multicolumn{2}{c}{kernel} & \multicolumn{2}{c}{DNN} \\ \hline
Augment training data & no  & yes & no  & yes\\ \hline
On validation &  0.97 & 0.77 &  0.71 &0.62  \\ \hline
On test  &   1.09 &  {\color{red}0.85} & {\color{red}0.69} & 0.77 \\ \hline
\end{tabular}
\end{table}
\subsection{Computer vision tasks}
\label{sImage}
We experiment on two problems:  handwritten digit recognition and object recognition.

\paragraph{Handwritten digit recognition}  We extract a dataset MNIST-6.7M, from the dataset MNIST-8M~\citep{loosli06selective}. MNIST-8M is a transformed version of the original MNIST dataset~\citep{lecun1998mnist}. Concretely, we randomly select 50,000 out of 60,000 images from  the MNIST's training set and extracted the corresponding samples (the original as well as transformed/distorted ones) in MNIST-8M, resulting in a total of 6.75 million samples in total, as our training set.  We use the remaining 10,000 images from the original training set as a validation set --- we purposely avoid using any transformed versions of those 10,000 images as a validation dataset to avoid potential overfitting. We report test error rate on the standard 10,000 MNIST test set.   

We also experimented with a data augmentation trick to  increase  the number of training samples \emph{during training}. Whenever we encounter a training sample, we  corrupt it with masking noise (randomly flipping 1 to 0 in the binary image).  We crudely tune the mask-out rates, which are either 0.1, 0.2 or 0.3.

Table~\ref{tMNISTSummary} compares the performance of the best single-kernel based classifier to that of the best DNN.  The kernel classifier uses Gaussian kernel, and 150,000 random projections.  The best DNN has 4 hidden layers with 1000, 2000, 2000, and 2000 hidden units respectively. 

The difference between the kernel model and the DNN is small -- about 16 (out of 10,000) misclassified images. Interestingly, on the test data, the kernel model benefits from the data augmentation trick while the DNN does not.  Possibly, the DNN overfits to the validation dataset.

\begin{table}
\small
\centering
\caption{Object recognition error (\%)}
\label{tCIFARSummary}
\vspace{0.5em}
\begin{tabular}{lcccc }
\hline
Model     & \multicolumn{2}{c}{kernel} & \multicolumn{2}{c}{DNN} \\ \hline
Augment training data & no  & yes & no  & yes\\ \hline
On validation  &  43.2 & 41.4 &  42.9 & 43.2  \\ \hline
On test &   43.9 &  {\color{red}42.2} &  {\color{red}43.3} & 44.0 \\ \hline
\end{tabular}
\end{table}

\paragraph{Object recognition} For this task, we experiment on the database CIFAR-10~\cite{krizhevsky2009learning}. The dataset contains 50,000 training samples and 10,000 test samples. Each sample is a RGB image of $32\times 32$ pixels in one of 10 object categories.  We randomly picked 10,000 images from the training set for validation, keeping the remaining 40,000 images for training. We did not perform any preprocessing to the images as we want to relate our findings to previously published results which often do not preprocess data or do not report all specific details in preprocessing.  We also experimented with an augmented version of the dataset, by injecting Gaussian noise to images during training.

Table~\ref{tCIFARSummary} compares the performance of the best single-kernel based classifier to  that of the best DNN.  The kernel classifier uses Gaussian kernel, and 4,000,000 random projections.  The best DNN has 3 hidden layers with 4000, 2000, 2000 hidden units respectively. 

The best kernel model performs slightly better than the DNN.  They both outperform previously reported DNN results on this dataset, whose error rates are between 44.4\% and 48.5\%~\citep{rifai2011contractive,krizhevsky2009learning,raiko2012deep}.  

Convolutional neural nets (CNNs) can significantly outperform DNNs on this dataset. However, we do not compare to CNNs  as our kernel models (as well as DNNs) do not construct feature extractions with prior knowledge, while  CNNs are  designed especially for object recognition.

% !TEX root = main.tex
\subsection{Automatic speech recognition (ASR)}
\label{sASR}
\paragraph{Task and evaluation metric} Deep neural nets (DNNs) have been very successfully applied to ASR. There, DNNs perform the task of acoustic modeling. Acoustic modeling is analogous to the conventional multi-class classification, that is, to learn a predictive model to assign phoneme context-dependent phoneme state labels $y$ to short segments of speech, called frames, represented as acoustic feature vectors $\vx$. Acoustic feature vectors are extracted from frames and their context windows (i.e., neighboring frames in temporal proximity). 

Analogously, kernel-based multinomial logistic regression models, as described in section~\ref{sRKS}, are also used as acoustic models and compared to DNNs.

Acoustic models are often evaluated in conjunction with other components of ASR systems. In particular, speech recognition is inherently a sequence recognition problem. Thus,  perplexity and classification accuracies --- commonly used for conventional multi-class classification problems ---  provide only a proxy (and intermediate goals) to the sequence recognition error. To measure the latter, a full ASR pipeline is necessary where the posterior probabilities of the phoneme states are combined with the  probabilities of the language models (of the interested linguistic units such as words) to yield the most probable sequence of those units. A best alignment with the ground-truth sequence is computed, yielding token error rates (\ter).

Given the inherent complexity, in what follows, we summarize the empirical studies of applying both paradigms to the acoustic modeling task. We will report \ter on two different languages. Details are presented in the Supplementary Material, including comparisons in terms of both perplexity and accuracy for different models.  We begin by describing the datasets, followed by a brief description of various kernel and DNN models we have experimented with.

\paragraph{Datasets}  We use two datasets:  the IARPA Babel Program
Cantonese (IARPA-babel101-v0.4c) and Bengali (IARPA-babel103b-v0.4b)
limited language packs.  Each pack contains a 20-hour training,  and a 20-hour test sets. 
We designate about 10\% of the training data as a held-out set to be used for model selection and tuning.

We follow the same procedure to extract acoustic features from raw audio data as in the previous work using DNNs for ASR~\citep{kingsbury13high}.  In particular, we have used IBM's proprietary system Attila which is adapted for the above-mentioned Babel language packs. The acoustic features are 360-dimensional real-valued dense vectors. There are 1000 (non-overlapping) phoneme context-dependent state labels for each language pack. For Cantonese, there are about 7.5 million data points for training, 0.9 million for held-out, and 7.2 million for test, and on Bengali,  7.7 million for training, 1.0 million for held-out and  7.1 million for test. For Bengali, the \ter metric  is the word-error-rate (WER) and for Cantonese, it is character-error-rate (CER).

%% Bengali: 7483896, 882588, 7197328
%% Cantonese: 7675795, 1047593, 7052334

\begin{table}
\small
\centering
\caption{Best Token Error Rates on Test Set ($\%$)}
\label{tTER}
\vspace{0.5em}
\begin{tabular}{ccc}\hline
Model  & Bengali & Cantonese \\ \hline
\ibm & 70.4  	&  	67.3	 \\ \hline
\rbm & {\color{red}69.5} 	& 	66.3 \\ \hline
\textbf{best kernel model} &  70.0	& 	{\color{red}65.7}	\\ \hline
\end{tabular}
\vspace{-1em}
\end{table}

\paragraph{Various of models being experimented} 
IBM's Attila ASR system has a DNN acoustic model that contains five hidden-layers, each of which contains 1,024 units with logistic nonlinearities.  We refer to this system as \ibm. We have developed another version of DNN, following the original Restricted Boltzman Machine (RBM)-based training  procedure for learning DNNs\citep{hinton06dbn}. Specifically, the pre-training is unsupervised.  We have trained DNNs with 1, 2, 3 and 4 hidden layers, and 500, 1000, and 2000 hidden units per layer (thus, 12 architectures per language).   We refer to this system as \rbm. 

For kernel-based acoustic models, we used Gaussian RBF, Laplacian kernels or some forms of combinations. The only hyper-parameter there to tune is the kernel bandwidth, which ranges from 0.3 - 5 median of the pairwise distances in the data (Typically, the median works well.), the number of random projections ranging from 2,000 to 400,000 (though a stable performance is often observed at 25,000 or above). For training with very large number of features, we used the parallel training procedure, described in section~\ref{sParallel}. For optimization, we used the stochastic average gradient and tune the step size loosely from 4 values $\{10^{-4}, 10^{-3}, 10^{-2}, 10^{-1}\}$.

Details about these systems are in Supplementary Material.

\paragraph{Results}
Table~\ref{tTER} reports the best performing models measured in \textsf{TER}. The RBM-trained DNN (\rbm ), which has 4 hidden layers and 2000 units in each layer, performs the best on Bengali. But our best kernel model, which uses Gaussian RBF kernel and 150,000 -- 200,000 random projections, performs the best on Cantonese.  Both perform better than IBM's DNN. On Cantonese, the improvement of the kernel model over \ibm is noticeably substantial ($1.6\%$ reduction in absolute).

% !TEX root = main.tex
\subsection{Computational efficiency}
\label{sTiming}

In contrast to DNNs, kernel models can be more efficiently developed. We illustrate this on two aspects: the computational cost of training a single model and the cost of model selection (i.e., hyperparameter tuning) 

\paragraph{Cost of training a single model} While the amount of training time depends on several factors, including the volume and the dimensionality of the dataset, the choice of hyperparameters and their effect on convergence, implementation details, etc. We give a rough picture after controlling those extraneous factors as much as possible.

We implement both methods with highly optimized Matlab codes (comparable to our CUDA C implementation) and utilize a single GPU (NVidia Tesla K20m). The timing results reported below are obtained from training acoustic models on the Bengali language dataset.

For a kernel model with 25,000 random projections (25 million model parameters), convergence is reached in less than 20 epochs, with an average of 15 minutes per epoch. In contrast, a competitive deep model with four hidden layers of 2,000 hidden units (15 million parameters), if initialized with pretrained parameters, reaches convergence in roughly 10 epochs, with an average of 28 minutes per epoch. (The pretraining requires additional 12 hours.)

Thus, the  training time for a single kernel model is about the same as that for a DNN. This holds for a range of  datasets and configurations of hyperparameters.

\paragraph{Cost of model selection} The number of kernel models to be tuned, is significantly (at least one order in magnitude ) less than DNNs.  There are only two hyperparameters to search when selecting kernel models: the kernel bandwidth and the learning rate. Generally, the higher the number of random projections, the better  the performance is. 

For DNNs, the number of hyperparameters needed to be tuned is substantially more. As previously mentioned, in our experiments, we tuned those related to the network architecture and optimization procedure, for both pre-training and fine-tuning. As such, it is fairly common to select the best DNN  among hundreds to thousands of them. 

Combining both factors, kernel models are especially appealing when they are used to tackle new problem settings where there is only a weak knowledge about what the optimal hyperparameters are or what the proper ranges for those parameters are. To develop DNNs in this scenario, one would be forced to combinatorially adjust many knobs while kernel approaches are simple and straightforward.
\begin{table}
\small
\centering
\caption{Combining best performing kernel and DNN models}
\label{tCombine}
\vspace{0.5em}
\begin{tabular}{ccccc}
\hline
Dataset & MNIST-6.7 & CIFAR-10 & Bengali & Cantonese\\ \hline
Best single & {\color{red}0.69} & {\color{blue}42.2} & {\color{red}69.5} & {\color{blue}65.7}\\ \hline
Combined & \textbf{0.61} & \textbf{40.3} & \textbf{69.1} & \textbf{64.9}\\ \hline
\end{tabular}
\end{table}
\subsection{Do kernel and deep models learn the same thing?}
\label{sComparison}

\begin{figure}
\centering
\includegraphics[width=0.45\columnwidth]{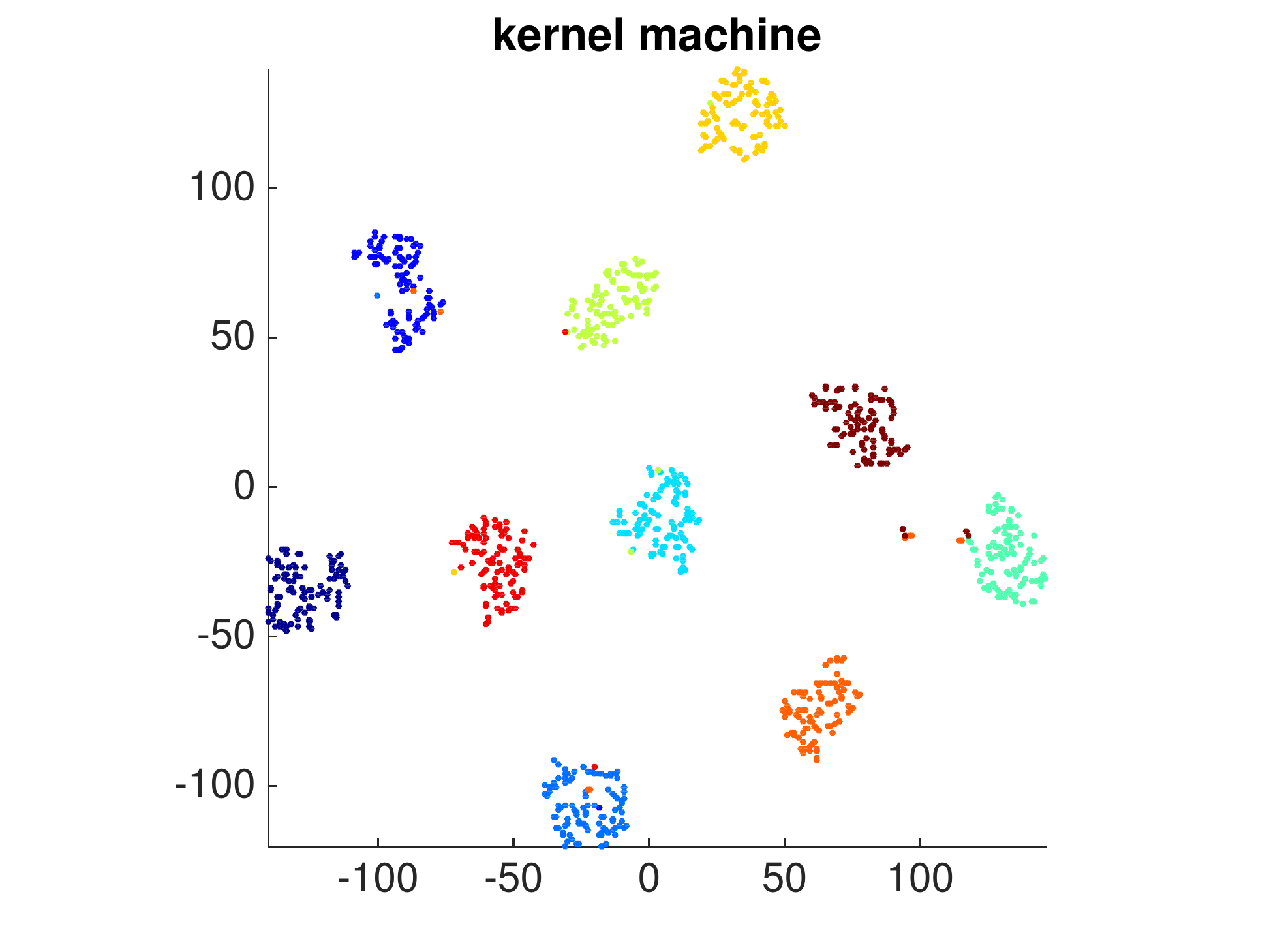}\includegraphics[width=0.45\columnwidth]{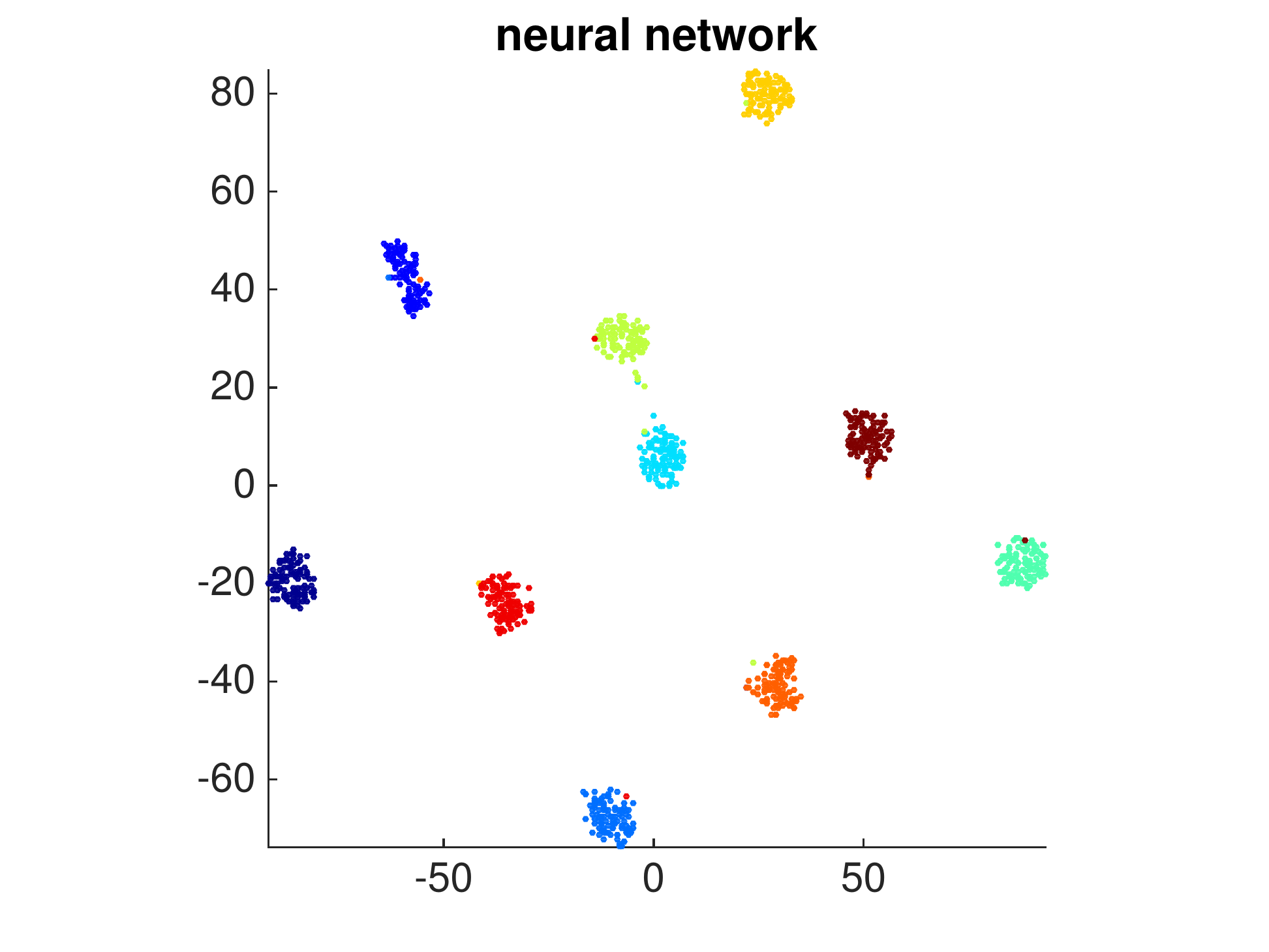}

\caption{t-SNE embeddings of data representation learnt by kernel (Left) and by DNN (Right). The relative positioning of those well-separated clusters are different, suggesting that these two models learn to represent data in quite different ways (Kernel embedding is computed using DNN's embedding as initialization to avoid the randomness in t-SNE.)}
\label{fMNIST}
\end{figure}

Given their matching performances, do kernel and DNN models learn the same knowledge from data?

We report in the following  our initial findings. We first combine the best performing models from each paradigm --- we use weighted sum of pre-softmax activations (i.e., logits). Table~\ref{tCombine} summarizes those results across 4 tasks we have studied in the previous sections. In the row of ``best single'', blue color indicates the number is from a kernel model while red from a DNN. Clearly, combining the two improves either one independently.

These improvements suggest that despite being close in error rates, the two models are still different enough. We gain more intuitive understanding by visualizing what are being learnt by each model. To this end, we project each model's pre-softmax activations  onto the 2-D plane with t-SNE. 

Fig.~\ref{fMNIST} contrasts the embeddings for 1000 samples from MNIST-6.7M's test set. Each data point is a dot and the color encodes its label. Due to the low classification error rates, it is not surprising to find that there are 10 well separated clusters, one for each of the 10 digits. However, the relative positioning of those clusters is noticeably different between DNN and kernel. It is not clear the embeddings can be transformed into each other with linear transformations. This seems to suggest each method has its unique way of \emph{nonlinearly} embedding data. Elucidating more precisely is our future research direction.

% !TEX root = main.tex
\section{Conclusion}
\label{sDiscuss}

We propose techniques to scale up kernel methods to large learning problems that are commonly found in speech recognition and computer vision.  We have shown that the performance of those large kernel models approaches or surpasses their deep neural networks counterparts, which have been regarded as the state-of-the-art. Future direction of our research include understanding the difference of these two camps of methods, for instance, in learning new representations of data.

% !TEX root = main.tex
\section*{Acknowledgement}

F. S. is grateful to Lawrence K. Saul (UCSD), L\'{e}on Bottou (Microsoft Research),  Alex Smola (CMU), and Chris J. C. Burges (Microsoft Research) for many fruitful discussions and pointers to relevant work. 

Computation for the work described in this paper was partially supported by the University of Southern CaliforniaÕs Center for High-Performance Computing (http://hpc.usc.edu). 

This work was supported by the Intelligence Advanced
Research Projects Activity (IARPA) via Department of Defense
U.S. Army Research Laboratory (DoD / ARL) contract
number W911NF-12-C-0012. The U.S. Government is authorized
to reproduce and distribute reprints for Governmental
purposes notwithstanding any copyright annotation thereon.
The views and conclusions contained herein are
those of the authors and should not be interpreted as necessarily
representing the official policies or endorsements, either
expressed or implied, of IARPA, DoD/ARL, or the U.S.
Government.

Additionally, A. B. G. is partially supported by a USC Provost Graduate Fellowship. F. S. is partially supported by a NSF IIS-1065243, a Google Research Award, an Alfred. P. Sloan Research Fellowship and an ARO YIP Award (W911NF-12-1-0241).

\small{
\bibliographystyle{plainnat}
\bibliography{global_strings_learning,main,crossref}

\begin{thebibliography}{51}
\providecommand{\natexlab}[1]{#1}
\providecommand{\url}[1]{\texttt{#1}}
\expandafter\ifx\csname urlstyle\endcsname\relax
  \providecommand{\doi}[1]{doi: #1}\else
  \providecommand{\doi}{doi: \begingroup \urlstyle{rm}\Url}\fi

\bibitem[Bach et~al.(2004)Bach, Lanckriet, and Jordan]{bach04mkl}
Francis~R. Bach, Gert R.~G. Lanckriet, and Michael~I. Jordan.
\newblock Multiple kernel learning, conic duality, and the {SMO} algorithm.
\newblock In \emph{Proc. of the 21th Intl. Conf. on Mach. Learn. (ICML)}, 2004.

\bibitem[Bengio et~al.(2009)Bengio, Schuurmans, Lafferty, Williams, and
  Culotta]{nips2009}
Y.~Bengio, D.~Schuurmans, J.D. Lafferty, C.K.I. Williams, and A.~Culotta,
  editors.
\newblock \emph{Advances in Neural Information Processing Systems 22}, 2009.

\bibitem[Bengio(2009)]{bengio09deep}
Yoshua Bengio.
\newblock Learning deep architectures for ai.
\newblock \emph{Foundations and Trends in Machine Learning}, 2\penalty0
  (1):\penalty0 1--127, January 2009.

\bibitem[Bengio et~al.(2013)Bengio, Courville, and
  Vincent]{bengio13representation}
Yoshua Bengio, Aaron~C. Courville, and Pascal Vincent.
\newblock Representation learning: a review and new perspectives.
\newblock \emph{IEEE Trans. on Pattern Anal. \& Mach. Intell.}, 35\penalty0
  (8):\penalty0 1798--1828, 2013.

\bibitem[Berg et~al.(1984)Berg, Christensen, and Ressel]{berg84}
Christian Berg, Jens Peter~Reus Christensen, and Paul Ressel.
\newblock \emph{Harmonic Analysis on Semigroups}.
\newblock Springer, 1984.

\bibitem[Bottou(2014)]{bottouXX}
L\'{e}on Bottou.
\newblock Personal communication, 2014.

\bibitem[Bottou and Lin(2007)]{bottou06svm}
L{\'{e}}on Bottou and Chih-Jen Lin.
\newblock Support vector machine solvers.
\newblock In  \citet{largescalekernelmachines07}.

\bibitem[Bottou et~al.(2007)Bottou, Chapelle, {DeCoste}, and
  Weston]{largescalekernelmachines07}
L\'{e}on Bottou, Olivier Chapelle, Dennis {DeCoste}, and Jason Weston, editors.
\newblock \emph{Large Scale Kernel Machines}.
\newblock MIT Press, Cambridge, MA., 2007.

\bibitem[Cheng and Kingsbury(2011)]{cheng11arcos}
Chih-Chieh Cheng and B.~Kingsbury.
\newblock Arccosine kernels: Acoustic modeling with infinite neural networks.
\newblock In \emph{Acoustics, Speech and Signal Processing (ICASSP), 2011 IEEE
  International Conference on}, pages 5200--5203, 2011.

\bibitem[Cho et~al.(2011)Cho, Ilin, and Raiko]{Cho11ICANN}
K.~Cho, A.~Ilin, and T.~Raiko.
\newblock Improved learning of gaussian-bernoulli restricted boltzmann
  machines.
\newblock In \emph{Proceedings of the International Conference on Artificial
  Neural Networks (ICANN 2011)}, pages 10--17, 2011.

\bibitem[Cho and Saul(2009)]{cho09arccos}
Youngmin Cho and Lawrence~K. Saul.
\newblock Kernel methods for deep learning.
\newblock In  \citet{nips2009}, pages 342--350.

\bibitem[Clarkson(2010)]{clarkson10coreset}
Kenneth~L. Clarkson.
\newblock Coresets, sparse greedy approximation, and the frank-wolfe algorithm.
\newblock \emph{ACM Trans. Algorithms}, 6\penalty0 (4):\penalty0 63:1--63:30,
  2010.

\bibitem[Cortes et~al.(2009)Cortes, Mohri, and Rostamizadeh]{cortes09nonlinear}
Corinna Cortes, Mehryar Mohri, and Afshin Rostamizadeh.
\newblock Learning non-linear combinations of kernels.
\newblock In  \citet{nips2009}, pages 396--404.

\bibitem[Cortes et~al.(2014)Cortes, Lawrence, and Weinberger]{nips2014}
Corinna Cortes, Neil Lawrence, and Kilian Weinberger, editors.
\newblock \emph{Advances in Neural Information Processing Systems 27}, 2014.

\bibitem[Dai et~al.(2014)Dai, Xie, He, Liang, Raj, Balcan, and
  Song]{dai14scalable}
Bo~Dai, Bo~Xie, Niao He, Yingyu Liang, Anant Raj, Maria-Florina Balcan, and
  Le~Song.
\newblock Scalable kernel methods via doubly stochastic gradients.
\newblock In  \citet{nips2014}.

\bibitem[Dasgupta and McAllester(2013)]{icml2013}
Sanjoy Dasgupta and David McAllester, editors.
\newblock \emph{Proc. of the 30th Int. Conf. on Mach. Learn. (ICML)}, volume~28
  of \emph{JMLR W \& CP}, 2013.

\bibitem[DeCoste and Sch\"{o}lkopf(2002)]{decoste02invariant}
Dennis DeCoste and Bernhard Sch\"{o}lkopf.
\newblock Training invariant support vector machines.
\newblock \emph{Mach. Learn.}, 46:\penalty0 161--190, 2002.

\bibitem[Deng et~al.(2012)Deng, T{\"{u}}r, He, and
  Hakkani{-}T{\"{u}}r]{deng12convex}
Li~Deng, G{\"{o}}khan T{\"{u}}r, Xiaodong He, and Dilek~Z. Hakkani{-}T{\"{u}}r.
\newblock Use of kernel deep convex networks and end-to-end learning for spoken
  language understanding.
\newblock In \emph{2012 {IEEE} Spoken Language Technology Workshop (SLT),
  Miami, FL, USA, December 2-5, 2012}, pages 210--215, 2012.

\bibitem[Hamid et~al.(2014)Hamid, Xiao, Gittens, and DeCoste]{hamid14compact}
Raffay Hamid, Ying Xiao, Alex Gittens, and Dennis DeCoste.
\newblock Compact random feature maps.
\newblock In  \citet{icml2013}, pages 19 -- 27.

\bibitem[Hinton(2002)]{hinton2002training}
G.~E Hinton.
\newblock Training products of experts by minimizing contrastive divergence.
\newblock \emph{Neural computation}, 14\penalty0 (8):\penalty0 1771--1800,
  2002.

\bibitem[Hinton et~al.(2012{\natexlab{a}})Hinton, Deng, Yu, Dahl, Mohamed,
  Jaitly, Senior, Vanhoucke, Nguyen, Sainath, and Kingsbury]{hinton12deep}
Geoffrey Hinton, Li~Deng, Dong Yu, George~E Dahl, Abdel-rahman Mohamed, Navdeep
  Jaitly, Andrew Senior, Vincent Vanhoucke, Patrick Nguyen, Tara Sainath, and
  Brian Kingsbury.
\newblock Deep neural networks for acoustic modeling in speech recognition: The
  shared views of four research groups.
\newblock \emph{Signal Processing Magazine, IEEE}, 29\penalty0 (6):\penalty0
  82--97, 2012{\natexlab{a}}.

\bibitem[Hinton et~al.(2006)Hinton, Osindero, and Teh]{hinton06dbn}
Geoffrey~E. Hinton, Simon Osindero, and Yee-Whye Teh.
\newblock A fast learning algorithm for deep belief nets.
\newblock \emph{Neual Comp.}, 18\penalty0 (7):\penalty0 1527--1554, 2006.

\bibitem[Hinton et~al.(2012{\natexlab{b}})Hinton, Srivastava, Krizhevsky,
  Sutskever, and Salakhutdinov]{hinton12dropout}
Geoffrey~E. Hinton, Nitish Srivastava, Alex Krizhevsky, Ilya Sutskever, and
  Ruslan~R. Salakhutdinov.
\newblock Improving neural networks by preventing co-adaptation of feature
  detectors.
\newblock arXiv:1207.0580, July 2012{\natexlab{b}}.
\newblock URL \url{http://arxiv.org/abs/1207.0580}.

\bibitem[Huang et~al.(2014)Huang, Avron, Sainath, Sindhwani, and
  Ramabhadran]{huang14kernel}
Po-Sen Huang, Haim Avron, Tara~N Sainath, Vikas Sindhwani, and Bhuvana
  Ramabhadran.
\newblock Kernel methods match deep neural networks on {TIMIT}.
\newblock In \emph{Proc. of the 2014 IEEE Intl. Conf. on Acou., Speech and Sig.
  Proc. (ICASSP)}, volume~1, page~6, 2014.

\bibitem[Kar and Karnick(2012)]{kar12random}
Purushottam Kar and Harish Karnick.
\newblock Random feature maps for dot product kernels.
\newblock In \emph{Proc. of the 29th Intl. Conf. on Mach. Learn. (ICML)}, 2012.

\bibitem[Kingsbury et~al.(2013)Kingsbury, Cui, Cui, Gales, Knill, Mamou, Mangu,
  Nolden, Picheny, Ramabhadran, Schl\"{u}ter, Sethy, and
  Woodland]{kingsbury13high}
Brian Kingsbury, Jia Cui, Xiaodong Cui, Mark J.~F. Gales, Kate Knill, Jonathan
  Mamou, Lidia Mangu, David Nolden, Michael Picheny, Bhuvana Ramabhadran, Ralf
  Schl\"{u}ter, Abhinav Sethy, and Phlip~C. Woodland.
\newblock A high-performance {C}antonese keyword search system.
\newblock In \emph{Proc. of the 2013 IEEE Intl. Conf. on Acou., Speech and Sig.
  Proc. (ICASSP)}, pages 8277--8281, 2013.

\bibitem[Kloft et~al.(2011)Kloft, Brefeld, Sonnenburg, and Zien]{kloft11lp}
Marius Kloft, Ulf Brefeld, S{\"{o}}ren Sonnenburg, and Alexander Zien.
\newblock \emph{l\({}_{\mbox{p}}\)}-norm multiple kernel learning.
\newblock \emph{Journal of Machine Learning Research}, 12:\penalty0 953--997,
  2011.

\bibitem[Krizhevsky and Hinton(2009)]{krizhevsky2009learning}
A.~Krizhevsky and G.~Hinton.
\newblock Learning multiple layers of features from tiny images, 2009.

\bibitem[Krizhevsky et~al.(2012)Krizhevsky, Sutskever, and
  Hinton]{krizhevsky12imagenet}
Alex Krizhevsky, Ilya Sutskever, and Geoffrey~E. Hinton.
\newblock Imagenet classification with deep convolutional neural networks.
\newblock In  \citet{nips2012}, pages 1097--1105.

\bibitem[Lanckriet et~al.(2004)Lanckriet, Cristianini, Bartlett, Ghaoui, and
  Jordan]{lanckriet04mkl}
Gert R.~G. Lanckriet, Nello Cristianini, Peter~L. Bartlett, Laurent~El Ghaoui,
  and Michael~I. Jordan.
\newblock Learning the kernel matrix with semidefinite programming.
\newblock \emph{Journal of Machine Learning Research}, 5:\penalty0 27--72,
  2004.

\bibitem[Le et~al.(2014)Le, Sarl{\'{o}}s, and Smola]{le13fastfood}
Quoc~Viet Le, Tam{\'{a}}s Sarl{\'{o}}s, and Alexander~Johannes Smola.
\newblock Fastfood: Approximating kernel expansions in loglinear time.
\newblock In  \citet{icml2013}.

\bibitem[LeCun and Cortes(1998)]{lecun1998mnist}
Y.~LeCun and C.~Cortes.
\newblock The mnist database of handwritten digits, 1998.

\bibitem[Loosli et~al.(2007)Loosli, Canu, and Bottou]{loosli06selective}
Ga\"{e}lle Loosli, St\'{e}phane Canu, and L\'{e}on Bottou.
\newblock Training invariant support vector machines using selective sampling.
\newblock In  \citet{largescalekernelmachines07}.

\bibitem[Mohamed et~al.(2012)Mohamed, Dahl, , and Hinton]{mohamed12dbn}
Abdel-rahman Mohamed, George Dahl, , and Geoffrey Hinton.
\newblock Acoustic modeling using deep belief networks.
\newblock \emph{IEEE Transactions on Audio, Speech, and Language Processing},
  20\penalty0 (1):\penalty0 14--22, 2012.

\bibitem[Neal(1994)]{neal94priors}
R.~Neal.
\newblock Priors for infinite networks.
\newblock Technical Report CRG-TR-94-1, Dept. of Computer Science, University
  of Toronto, 1994.

\bibitem[Pereira et~al.(2012)Pereira, Burges, Bottou, and Weinberger]{nips2012}
F.~Pereira, C.J.C. Burges, L.~Bottou, and K.~Q. Weinberger, editors.
\newblock \emph{Advances in Neural Information Processing Systems 25}, 2012.

\bibitem[Platt(1998)]{platt98smo}
John~C. Platt.
\newblock Fast training of support vector machines using sequential minimal
  optimization.
\newblock In \emph{Advances in Kernel Methods - Support Vector Learning}. MIT
  Press, 1998.

\bibitem[Rahimi and Recht(2007)]{rahimi07random}
Ali Rahimi and Benjamin Recht.
\newblock Random features for large-scale kernel machines.
\newblock In \emph{Advances in Neural Information Processing Systems 20}, pages
  1177--1184, 2007.

\bibitem[Rahimi and Recht(2008)]{rahimi08kitchen}
Ali Rahimi and Benjamin Recht.
\newblock Weighted sums of random kitchen sinks: Replacing minimization with
  randomization in learning.
\newblock In \emph{Advances in Neural Information Processing Systems 21}, pages
  1313--1320, 2008.

\bibitem[Raiko et~al.(2012)Raiko, Valpola, and LeCun]{raiko2012deep}
T.~Raiko, H.~Valpola, and Y.~LeCun.
\newblock Deep learning made easier by linear transformations in perceptrons.
\newblock In \emph{International Conference on Artificial Intelligence and
  Statistics}, pages 924--932, 2012.

\bibitem[Rifai et~al.(2011)Rifai, Vincent, Muller, Glorot, and
  Bengio]{rifai2011contractive}
S.~Rifai, P.~Vincent, X.~Muller, X.~Glorot, and Y.~Bengio.
\newblock Contractive auto-encoders: Explicit invariance during feature
  extraction.
\newblock In \emph{Proceedings of the 28th International Conference on Machine
  Learning (ICML-11)}, pages 833--840, 2011.

\bibitem[Roux et~al.(2012)Roux, Schmidt, and Bach]{roux12sag}
Nicolas~L. Roux, Mark Schmidt, and Francis~R. Bach.
\newblock A stochastic gradient method with an exponential convergence rate for
  finite training sets.
\newblock In  \citet{nips2012}, pages 2663--2671.

\bibitem[Sch\"{o}lkopf and Smola(2002)]{scholkopf02}
B.~Sch\"{o}lkopf and A.~Smola.
\newblock \emph{Learning with kernels}.
\newblock MIT Press, 2002.

\bibitem[Seide et~al.(2011{\natexlab{a}})Seide, Li, Chen, and Yu]{seide11dnn}
Frank Seide, Gang Li, Xie Chen, and Dong Yu.
\newblock Feature engineering in context-dependent deep neural networks for
  conversational speech transcription.
\newblock In \emph{Automatic Speech Recognition and Understanding (ASRU), 2011
  IEEE Workshop on}, pages 24--29, 2011{\natexlab{a}}.

\bibitem[Seide et~al.(2011{\natexlab{b}})Seide, Li, and Yu]{seide11cddnn}
Frank Seide, Gang Li, and Dong Yu.
\newblock Conversational speech transcription using context-dependent deep
  neural networks.
\newblock In \emph{Proc. of Interspeech}, pages 437--440, 2011{\natexlab{b}}.

\bibitem[Smola(2014)]{smolaXX}
Alex Smola.
\newblock Personal communication, 2014.

\bibitem[Sonnenburg and Franc(2010)]{sonnenburg10coffin}
S\"{o}ren Sonnenburg and Vojtech Franc.
\newblock {COFFIN:} {A} computational framework for linear {SVM}s.
\newblock In \emph{Proc. of the 27th Intl. Conf. on Mach. Learn. (ICML)}, pages
  999--1006, Haifa, Israel, 2010.
\newblock URL \url{http://www.icml2010.org/papers/280.pdf}.

\bibitem[Szegedy et~al.(2014)Szegedy, Liu, Jia, Sermanet, Reed, Anguelov,
  Erhan, Vanhoucke, and Rabinovich]{szegedy14googlelenet}
Christian Szegedy, Wei Liu, Yangqing Jia, Pierre Sermanet, Scott Reed, Dragomir
  Anguelov, Dumitru Erhan, Vincent Vanhoucke, and Andrew Rabinovich.
\newblock Going deeper with convolutions.
\newblock In  \citet{nips2014}.

\bibitem[Tsang et~al.(2005)Tsang, Kwok, and Cheung]{tsang05cvm}
Ivor~W. Tsang, James~T. Kwok, and Pak{-}Ming Cheung.
\newblock Core vector machines: Fast {SVM} training on very large data sets.
\newblock \emph{Journal of Machine Learning Research}, 6:\penalty0 363--392,
  2005.
\newblock URL \url{http://www.jmlr.org/papers/v6/tsang05a.html}.

\bibitem[Vedaldi and Zisserman(2012)]{vedaldi12additive}
A.~Vedaldi and A.~Zisserman.
\newblock Efficient additive kernels via explicit feature maps.
\newblock \emph{IEEE Trans. on Pattern Anal. \& Mach. Intell.}, 34\penalty0
  (3):\penalty0 480--492, 2012.

\bibitem[Williams(1996)]{williams96infinite}
C.~K.~I. Williams.
\newblock Computing with infinite networks.
\newblock In \emph{Advances in Neural Information Processing Systems 19}, pages
  599--621, 1996.

\end{thebibliography}
}
\appendix
% !TEX root = main.tex
\section{Proof of the Theorem 2}
\setcounter{thm}{1}
\begin{thm}
Suppose all $k_i(\cdot, \cdot)$ are translation-invariant kernels such that
\begin{equation*}
k_i(\vx - \vz)  = \int_{R^d} p_i(\vomega) e^{j\vomega\T(\vx-\vz)}\, d\vomega 
\end{equation*}
Then $k(\cdot, \cdot)  = \prod_{i=1}^L k_i (\cdot, \cdot)$ is also translation-invariant such that
\begin{equation*}
k(\vx - \vz)  = \int_{R^d} p(\vomega) e^{j\vomega\T(\vx-\vz)}\, d\vomega 
\end{equation*}
where the probability measure $p(\vomega)$ is given by the convolution of $p_i(\vomega)$s
\begin{equation*}
p(\vomega) = p_1(\vomega) \ast p_2(\vomega)  \ast \cdots \ast p_L(\vomega)
\end{equation*}
Moreover, let $\vomega_i \sim p_i(\vomega)$ be a random variable drawn from the corresponding distribution, then 
\begin{equation*}
\vomega = \sum_i \vomega_i \sim p(\vomega)
\end{equation*}
Namely, to approximate $k(\cdot, \cdot)$, one needs only to draw random variables from each individual component kernel's corresponding density, and use the sum of those variables to compute random features.
\end{thm}

\begin{proof}
Denote $\vDD = \vx-\vz$.

For translation-invariant kernel, we have
\begin{eqnarray*}
\tiny
\setlength{\arraycolsep}{1em}
k_i(\vx,\vz) = k_i(\vDD) & = & 
\int_{\vomega_i} p_i(\vomega_i)e^{j\vomega_i\T \vDD} \mathrm{d} \vomega_i
\end{eqnarray*}
The product of the kernels is,

\begin{eqnarray*}
k(\cdot, \cdot)  =  
\prod_{i=1}^L k_i (\cdot, \cdot)  = 
\prod_{i=1}^L k_i (\vDD) = k(\vDD),
\end{eqnarray*}
which is also translation-invariant.
\begin{eqnarray*}
&  k(\vDD) = & \prod_{i=1}^L \int_{\vomega_i}  p_i(\vomega_i)e^{j\vomega_i\T \vDD} \mathrm{d} \vomega_i  \\
& = & \!\! 
 \int\limits_{\vomega_1 \! \mathellipsis\! \vomega_L} p_1(\vomega_1) \! \mathellipsis \! p_L(\vomega_L)
 e^{j (\sum_i^L\vomega_i\T) \vDD} 
%\exp \lbrack  j (\sum_i^L\vomega_i\T) \vDD \rbrack
 \mathrm{d} \vomega_1\! \mathellipsis\! \mathrm{d} \vomega_L \\
& \stackrel{\tilde{\omega} = \sum_i^L\vomega_i}{=} &
\int\limits_{\tilde{\vomega}} 
\left[  \int\limits_{\vomega_1 \mathellipsis \vomega_{L-1}} \! p_1(\vomega_{1})\ldots p_{L-1}(\vomega_{L-1})  \right.
\\ && \left. \qquad p_L(\tilde{\vomega} - \sum_i^{L-1}\vomega_i) \mathrm{d} \vomega_1\! \mathellipsis\! \mathrm{d} \vomega_{L-1}  \right] 
e^{j \tilde{\vomega}\T \Delta} \mathrm{d} \tilde{\vomega} \\
&= & \int p_{\tilde{\vomega}} (\tilde {\vomega})e^{j \tilde{\vomega}\T \Delta} \mathrm{d} \tilde{\vomega}  \\
& = & \int p_{\tilde{\omega}} (\tilde {\omega})e^{j \tilde{\vomega}\T (\vx-\vz)} \mathrm{d} \tilde{\omega}  \\
& = & \expect{\vct{\tilde{\omega}}}{\phi_{\tilde{\omega}}(\vx)\phi_{\tilde{\omega}}(\vz)^*}
%\end{flalign*}
\end{eqnarray*}

We have used the fact (due to convolution theorem) that
\begin{eqnarray*}
&& \int\limits_{\vomega_1 \! \mathellipsis \vomega_{L-1}}   p_1(\vomega_{1})\ldots p_{L-1}(\vomega_{L-1}) \\
&& \quad p_L(\tilde{\vomega} - \sum_i^{L-1}\vomega_i) \mathrm{d} \vomega_1 \! \mathellipsis\! \mathrm{d} \vomega_{L-1}  \\
& =  & p_1(\vomega_1) \ast p_2(\vomega_2)   \ast \cdots \ast p_L(\vomega_L) = p_{\tilde{\vomega}} (\tilde {\vomega})
\end{eqnarray*}
It means we have found a new distribution $p_{\tilde{\vomega}} (\tilde {\vomega})$ as the random projection generating distribution for the new kernel $$k(\cdot, \cdot) = \prod_{i=1}^L k_i (\cdot, \cdot).$$ 

From the definition of $\tilde{\omega}$, in order to sample from $p_{\tilde{\vomega}} (\tilde {\vomega})$, 
we can simply use the sum of independent samples from $p_i(\vomega_i)$.
\end{proof}

% !TEX root = main.tex
\section{Detailed experimental Results}
%\paragraph{General Setup}
\subsection{Image recognition}

We first provide details on our empirical studies on challenging problems in image recognition.

\subsubsection{Handwritten digit recognition}
\paragraph{Dataset and Preprocessing}
Dataset is described in section 5.2 of the main text. We scale the input between $[0,1)$ by dividing 256.
\paragraph{Kernel} We use Gaussian RBF and Laplacian kernels with kernel bandwidth selected from  $\{1, 1.5, 2, 2.5, 3\}\times$ median of the pairwise distance in data.  We select the learning rate from \{$5\times 10^{-4}, 10^{-3}, 5\times 10^{-3},10^{-2},5\times 10^{-2},10^{-1}$\}. The random feature dimension we have used is 150,000. Performance with different dimension is shown in Table \ref{sTkernelMnist}.
\paragraph{DNN}
We trained DNNs with $1, 2, 3, $ or $4$ hidden layers, with 1000, 2000, 2000 and 2000 hidden units respectively. We firstly pre-trained 1 Gaussian-Bernoulli and 3  consecutive Bernoulli restricted Boltzmann machines (RBMs), all using Stochastic Gradient Descent (SGD) with Contrastive Divergence (CD-1) algorithm \citep{hinton2002training}. 

We select learning  rate from \{$10^{-1}, 1.5\times10^{-1}$\}, momentum from \{0.5, 0.9\} and set L2 regularization to  $2\times10^{-4}$ for 2 epochs of pretraining. 
In finetuning, we tune SGD with learning rate from \{$5\times10^{-3}, 10^{-2}, 5\times10^{-2}, 10^{-1}, 5\times10^{-1}$\}, momentum from \{0.7, 0.9\}. We decrease the learning rate by a factor of 0.99 for every epoch and set mini-batch size to 100, L2 regularization to 0. We use early-stopping to control overfitting. When trained with data augmentation, we use smaller learning rate and run for more epochs.
\paragraph{Data Augmentation} We use mask-out noise with ratio \{0.1, 0.2, 0.3\} for both kernel methods and DNN.
\paragraph{Results}
Table \ref{tMNISTall} compares the performance of kernel methods to deep neural nets of different architectures. The best result of DNN is a 4-hidden-layer neural network. Deep nets generally have slightly smaller test errors. Kernel models benefit more from data augmentation and achieve similar error rates.

\begin{table}[h]
\small
\centering
\caption{Kernel Methods on MNIST-6.7M (error rates \%) }
\label{sTkernelMnist}
\vspace{0.5em}
\begin{tabular}{|c|c|c|c|c|c|c|}
\hline
Kernel type & Data aug. & 10K       & 50K       & 100K      & 150K          \\ \hline
Gaussian  & No & 1.45/1.42 & 1.03/1.25 & 0.98/1.12 & 0.97/1.09 \\ \hline
Laplacian & No & 1.93/1.93 & 1.21/1.34 & 1.16/1.17 & 1.10/1.13 \\ \hline
Gaussian  & Yes & - & 0.83/1.03  & 0.79/0.92 & 0.77/0.85 \\ \hline
\end{tabular}
\end{table}

\begin{table}[h]

\small
\centering
\caption{DNN on MNIST-6.7M (error rates \%) }
\label{tMNISTall}
\vspace{0.5em}
\begin{tabular}{c|cccc }
\hline
Model     & \multicolumn{2}{c}{Original} & \multicolumn{2}{c}{Augmented} \\ \hline
	      & Validation & Test & Validation  & Test\\ \hline
kernel & 0.97	&   1.09  &  0.77 & 0.85  \\ \hline
4 hidden & 0.71	&   \textbf{0.69}  &  0.64 & 0.80  \\ \hline
3 hidden &  0.78 & 0.73 &  0.74  & 0.77  \\ \hline
2 hidden  &   0.76 & 0.71 & 0.64 & 0.79 \\ \hline
1 hidden & 0.84 & 0.95 & 0.79 & \textbf{0.76} \\ \hline
\end{tabular}
\end{table}

\paragraph{PCA Embedding}
In addition to t-SNE visualization, we also project each model's pre-softmax activation onto the first two principle components given by PCA. Fig.~\ref{fMNIST2} contrasts the PCA embeddings for 1000 samples from MNIST-6.7M's test set. 
\begin{figure}
\centering
\includegraphics[width=0.49\columnwidth]{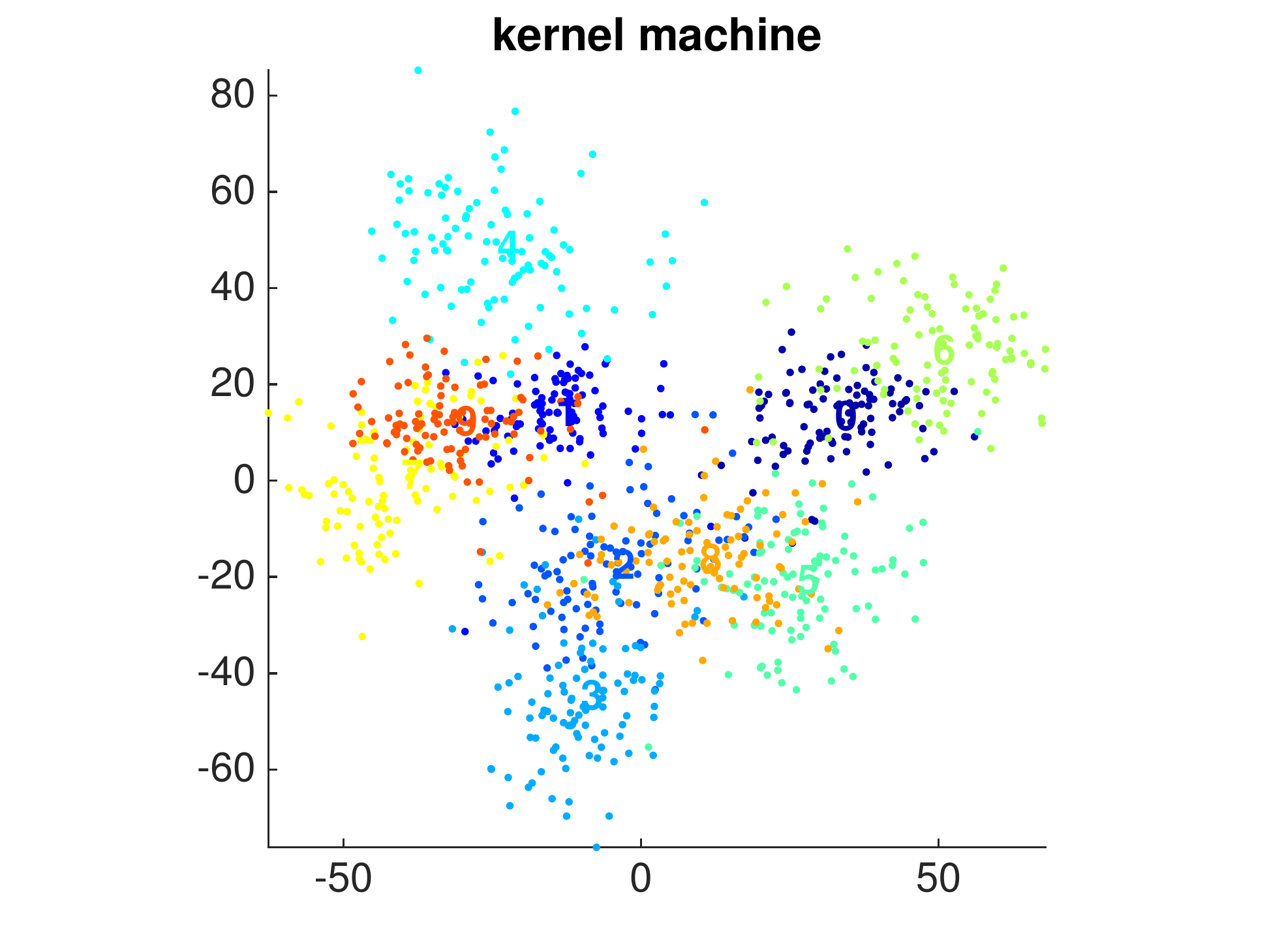}\includegraphics[width=0.49\columnwidth]{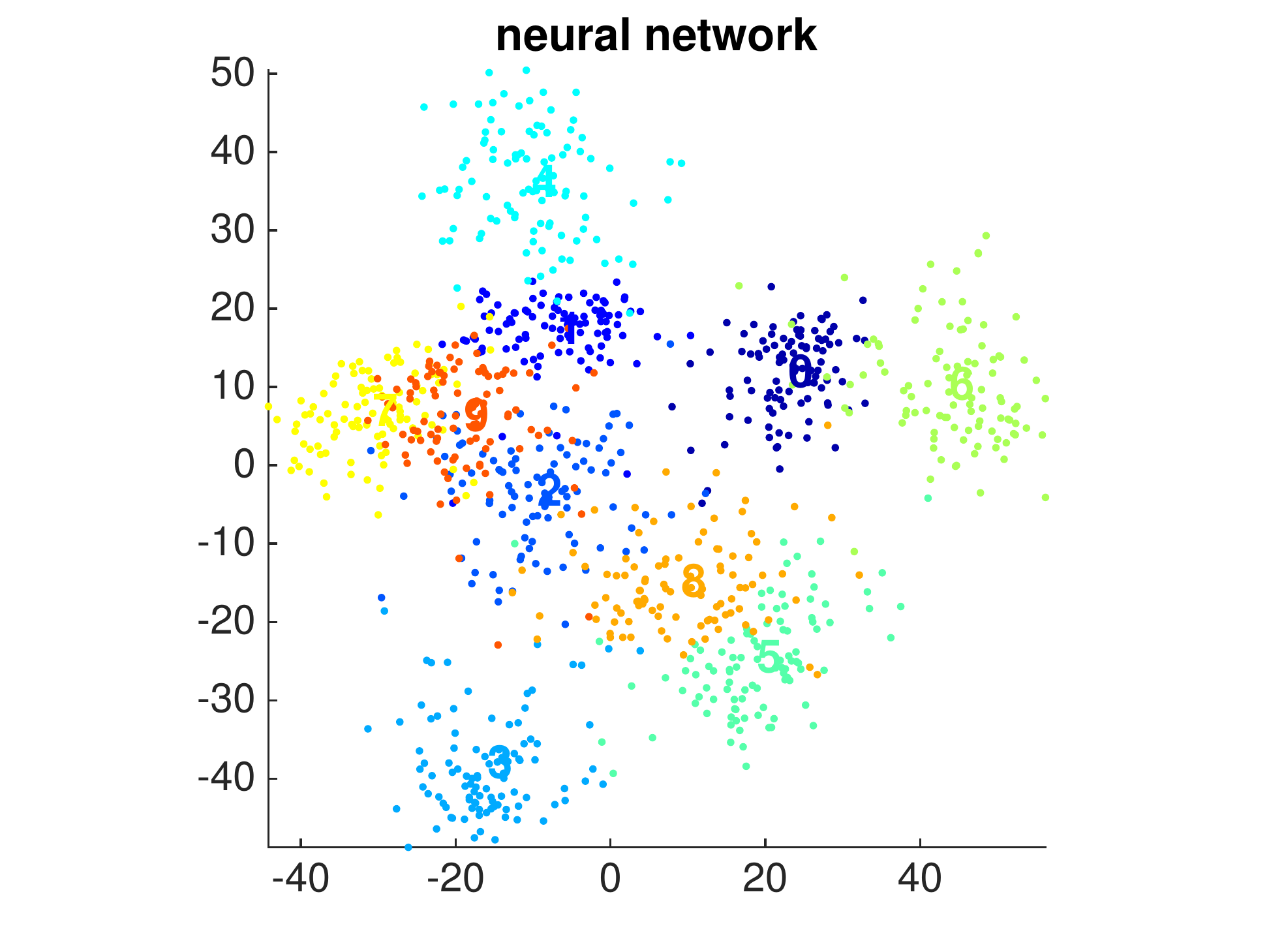}
\caption{Embeddings of data by kernel (Left) and by DNN (Rights).}
\label{fMNIST2}
\vspace{-1em}
\end{figure}
Neural network seems to give more spread out embeddings compared to kernel machine. The most noticeable class is digit 3 (light blue in lower left corner) and digit 6 (green on the right).

\subsubsection{Object Recognition}
\paragraph{Dataset and Preprocessing} Dataset is described in section 5.2 of the main text. We scale the input between $[0,1)$ by dividing 256. No other preprocessing is used because we would like to relate to previously reported results on DNNs on this data where preprocessing is not applied. 

\paragraph{Kernel} Gaussian kernel is used. We achieve the best performance by using 4,000,000 random features. This is done by training 200K single models in parallel and then combine. Table \ref{tTkernelCifar} shows the performance of kernel with respect to different number of random features. Similarly, we select bandwidth from  \{0.5,1,1.5,2,3\} $\times$ median distance and learning rate from $\{5\times 10^{-4}, 10^{-3}, 5\times 10^{-3},10^{-2},5\times 10^{-2},10^{-1} \}$.

\paragraph{DNN}
We trained DNNs with 1 to 4 hidden layers, with 2000  hidden units per layer. In pretraining, we use a Gaussian RBM for the input layer and three Bernoulli RBMs for  intermediate hidden layers using CD-1 algorithm. (We adopt the parameterization of GRBM in \cite{Cho11ICANN}, which shows better performance. )

For Gaussian RBM, we tune learning rate from \{$5\times 10^{-5}, 10^{-4}, 2\times10^{-4}$\}, momentum from \{0.2, 0.5, 0.9\}\footnote{Momentum can be increased to another choice from the three after 5 epochs.}, and L2 regularization from \{$2\times10^{-5}, 2\times10^{-4}$\}.  For Bernoulli RBM, we tune learning rate from \{$10^{-2}, 2.5\times10^{-2}, 5\times10^{-2}$\}, momentum from \{0.2, 0.5, 0.9\}, and L2 regularization is \{$2\times10^{-5}$\}.
In finetuning, we tune SGD with learning rate from \{$4\times 10^{-2}, 8 \times10^{-2}, 1.2\times10^{-1}$\}, momentum from \{0.2, 0.5, 0.9\}, decrease the learning rate by a factor of 0.9 for every 20 or 50 epochs and set mini-batch size to 50, L2 regularization to 0. We use early-stopping to control overfitting. When trained with data augmentation, we use smaller learning rate and run for more epochs.

We used constant learning rate throughout 30 epochs and update momentum after 5 epochs. In finetuning stage, we used stochastic gradient descent with 0.9 momentum, fixed learning rate schedule, decreasing the learning rate by 10 after 50 epochs. The optimal model is selected according to the classification accuracy on validation dataset. We trained DNNs with 1, 2, 3 and 4 hidden layers with 2000 hiddenunits per layer. Overfitting was observed after we increased model from 3 hidden layers to 4 hidden layers. 

\paragraph{Data Augmentation} We apply additive Gaussian noise with standard deviation \{0.1, 0.2, 0.3\} on raw pixel for both kernel methods and DNN.
\paragraph{Results}
Table \ref{tCIFARall} contrasts results of kernel models and DNN. We observe overfitting for 4 hidden layer neural network, and achieve best results in a 3 hidden layer architecture. Deeper models start to overfit and give worse validation and test performance. Kernel models achieve the best error when data augmentation is used.

\begin{table}
\small
\centering
\caption{Kernel Methods on CIFAR-10 (error rate\%) }
\label{tTkernelCifar}
\vspace{0.5em}
{
\begin{tabular}{c|cccc}
\hline
Gaussian r.f.    & \multicolumn{2}{c}{Original} & \multicolumn{2}{c}{Augmented} \\ \hline
	      & Validation & Test & Validation  & Test\\ \hline
200K & 43.74 &  44.48  & 42.15 & 43.13  \\ \hline
1M  & 43.43 & 44.08 & 41.62 &  42.38 \\ \hline
2M   & 43.26 & 44.04 & 41.47 & 42.26\\ \hline
4M  & 43.22 & \textbf{43.93} & 41.36 & \textbf{42.23} \\ \hline
\end{tabular}

}
\end{table}

\begin{table}
\small
\centering
\caption{DNN on CIFAR-10  (error rate \%)}
\label{tCIFARall}
\vspace{0.5em}

\begin{tabular}{c|cccc }
\hline
Model     & \multicolumn{2}{c}{Original} & \multicolumn{2}{c}{Augmented} \\ \hline
	      & Validation & Test & Validation  & Test\\ \hline
kernel & 	43.22 &  43.93  & 41.36 & \textbf{42.23}  \\ \hline
4 hidden & 43.21 &  43.74  & 43.00 & 43.38  \\ \hline
3 hidden & 42.89 & \textbf{43.29} & 42.93 &  43.35 \\ \hline
2 hidden  & 43.30 & 43.76 & 43.80 & 44.81\\ \hline
1 hidden & 48.40 & 48.94 & 47.28 & 47.79 \\ \hline
\end{tabular}
\end{table}

% !TEX root = main.tex
\subsection{Automatic speech recognition}

In what follows, we provide comprehensive details on our empirical studies on challenging problems in automatic speech recognition.

\subsubsection{Tasks, datasets and evaluation metrics} 

\paragraph{Task}\ We have selected the task of acoustic modeling, a crucial component in automatic speech recognition. In its most basic form,  acoustic modeling is analogous to the conventional multi-class classification, that is, to learn a predictive model to assign phoneme context-dependent state labels to short segments of speech, called frames. While speech signals are highly non-stationary and context-sensitive, acoustic modeling addresses this issue by using acoustic features extracted from context windows (i.e., neighboring frames in temporal proximity) to capture the transient characteristics of the signals. 

\paragraph{Data characteristics}\ To this end, we use two datasets:  the IARPA Babel Program
Cantonese (IARPA-babel101-v0.4c) and Bengali (IARPA-babel103b-v0.4b)
limited language packs.  Each pack contains a 20-hour training,  and a 20-hour test sets. 
We designate about 10\% of the training data as a held-out set to be used for model selection and tuning (i.e., tuning hyperparameters etc).  The training, held-out, and test sets contain different speakers.  The acoustic data is very challenging as it is two-person conversations between people who
know each other well (family and friends) recorded over telephone
channels (in most cases with mobile telephones) from speakers in a
wide variety of acoustic environments, including moving vehicles and
public places.  As a result, it contains many natural phenomena such
as mispronunciations, disfluencies, laughter, rapid speech, background
noise, and channel variability. Compared to the more familiar TIMIT corpus, which contains about 4 hours of training data, the Babel data is substantially more
challenging because  the TIMIT data is read
speech recorded in a well-controlled, quiet studio environment.

As is standard on previous work using DNNs for speech recognition, the
data is preprocessed using Gaussian mixture models to give
alignments between phoneme state labels and 10-millisecond-frames of
speech~\citep{kingsbury13high}. The acoustic features are 360-dimensional real-valued dense vectors. There are 1000 (non-overlapping) phoneme context-dependent state labels for each language pack. For Cantonese, there are about 7.5 million data points for training, 0.9 million for held-out, and 7.2 million for test, and on Bengali,  7.7 million for training, 1.0 million for held-out and  7.1 million for test. 

%% Bengali: 7483896, 882588, 7197328
%% Cantonese: 7675795, 1047593, 7052334

\paragraph{Evaluation metrics}\ We will be reporting 3 evaluation metrics, typically found in mainstream speech recognition research.\\[1em]

\noindent
\textsf{Perplexity}\ Given a set of examples, $\{ (\vx_i, y_i), i = 1 \ldots m\}$, the perplexity
is defined as
\[
\textsf{perp} = \exp\left\{-\frac{1}{m} \sum_{i=1}^m \log p(y_i | \vct{x}_i)\right\}
\]
The perplexity measure is lower bound by 1 when all predictions are perfect: $p(y_i | \vx_i) = 1$ for all samples. With random guessing $p(y_i |
\vx_i) = 1/C$, where $C$ is the number of classes, the perplexity attains $C$.  

We use the perplexity measure on the held-out for model selection and tuning. This is because the perplexity is often found to be correlated with the next two performance measures.\\[1em]

\noindent
\textsf{Accuracy}\ The classification accuracy is defined as 
\[
 \textsf{acc} =  \frac{1}{m} \sum_{i=1}^m \mathbbm{1} \left[ y_i = \argmax_{y \in {1,2, \ldots, C}} p(y|\vct{x}_i)\right] 
\]

\noindent
\textsf{Token Error Rate (TER)} Speech recognition is inherently a sequence recognition problem. Thus, \textsf{perp} and \textsf{acc} provide only proxy (and intermediate goals) to the sequence recognition error. To measure the latter, a full automatic speech recognition pipeline is necessary where the posterior probabilities of the phoneme labels $p(y|\vx)$ are combined with the  probabilities of the language models (of the interested linguistic units such as words) to yield the most probable sequence of those units.  A best alignment with the ground-truth sequence is computed, yielding token error rates.
For Bengali, the token error rate is the word-error-rate (WER) and for Cantonese, it is character-error-rate (CER). 

Because it entails performing speech recognition, obtaining \textsf{TER} is computationally costly thus it is rarely used for model selection and tuning.  Note also that the token error rates obtained on the Babel tasks are
much higher than those are reported for other conversational speech tasks
such as Switchboard or Broadcast News.  This is because we have much less training data for Babel than for the other tasks.  This {\em low-resource}\/
setting is an important one in the speech processing area, given that
there are a large number of languages in the world for which speech and language models do
not currently exist.

\subsubsection{Deep neural nets acoustic models}

There are many variants of DNNs techniques. We have decided to choose two flavors that are very different in learning from data, in order to have a broader comparison. In either case, our model tuning is extensive.

\paragraph{IBM's DNN}  We have used IBM's proprietary system Attila for the conventional speech recognition that is adapted for the above-mentioned Babel task. A detailed description appears in~\cite{kingsbury13high}. Attila contains a state-of-the-art acoustic model  provided by IBM. It also powers our full ASR pipeline in order to compute token error rate (\textsf{TER}). We have also used it to convert raw speech signals into acoustic features. Concretely, the features at a frame is a 40-dimensional speaker-adapted representation that has previously been shown to work well with DNN acoustic models~\cite{kingsbury13high}. Features at 8 neighboring contextual frames are concatenated, yield 360-dimensional features.  We have used the same features for our kernel methods.

IBM's DNN acoustic model contains five hidden-layers, each of which contains 1,024 units with logistic nonlinearities. The output is a softmax nonlinearity with 1,000 targets that correspond to quinphone context-dependent HMM states clustered using decision trees. All layers in the DNN are fully connected. The training of the DNN occurs in two stages. First, a greedy layer-wise discriminative pretraining ~\citep{seide11cddnn} to set the weights for each layer in a reasonable range. Then, the cross-entropy criterion is minimized with respect to all parameters in the network, using  stochastic gradient descent with a mini-batch size of 250 samples, without momentum, and with annealing the learning rate based on the reduction in cross-entropy loss on a held-out set. 

\paragraph{RBM-DNN}  We have designed another version of DNN, following the original Restricted Boltzman Machine (RBM)-based training  procedure for learning DNNs\citep{hinton06dbn}. Specifically, the pre-training is unsupervised.  We have trained DNNs with 1, 2, 3 and 4 hidden layers, and 500, 1000, and 2000 hidden units per layer (thus, totally 12 architectures per language).

The first hidden layer is a Gaussian RBM and the upper layers are Binary-Bernoulli RBM. In pre-training, we use 5 epochs of SGD with Contrastive Divergence (CD-1) algorithm  on all training data. We tuned 3 hyper-parameters, which are learning rate, momentum, and the strength for an $\ell_2$ regularizer. For fine-tuning, we used error back-propagation. We tuned the initial learning rate, learning rate decay, momentum and the strength for another $\ell_2$ regularizer.   The fine-tuning usually converges in 10 epochs.

\subsubsection{Kernel acoustic models}

The development of kernel acoustic models does not require combinatory searching over many factors. We experimented only two types of kernels: Gaussian RBF and Laplacian kernels. The only hyper-parameter there to tune is the kernel bandwidth, which ranges from 0.3 - 5 median of the pairwise distances in the data. (Typically, the median works well.)

The random feature dimensions we have used ranging from 2,000 to 400,000, though a stable performance is often observed at 25,000 or above. For training with very large number of features, we used the parallel training procedure, described in section 4.1 of the main text.

All kernel acoustic models are multi-nomial logistic regression, thus optimized by convex optimization. As mentioned in section 4.1 of the main text, we use Stochastic Average Gradient (SAG), which efficiently leverages the convexity property. We do tune the step size, selected from a loose range of 4 values $\{10^{-4}, 10^{-3}, 10^{-2}, 10^{-1}\}$.

For additive and multiplicative kernel combinations,  we combine only two, one Gaussian and the other Laplacian. For additive combinations,  we first train two models, one for each kernel. The combining coefficient $\alpha$ is selected from $0.1, 0.2, \ldots, 0.9$. For composite kernels, we compose Gaussian with Laplacian. We perform a supervised dimensionality reduction, as described in section 4.2 of the main text. The reduced dimensionality is chosen from 50, 100, or 360. The first kernel's bandwidth is greedily selected to be optimal as a single-kernel acoustic model. The other kernel's bandwidth is selected after composing the features.

\begin{table}

\centering
\caption{Best perplexity and accuracy by different models (see texts for description of different models)}
\label{tPerpAcc}
\begin{tabular}{|c|c|c|c|c|}\hline
& \multicolumn{2}{|c|}{Bengali} & \multicolumn{2}{|c|}{Cantonese}  \\ \hline
Model   & \textsf{perp}  & \textsf{acc} (\%) & \textsf{perp}  & \textsf{acc} (\%) \\ \hline \hline
\ibm  &   3.4/3.5 	&	71.5/71.2  	 	& 6.8/6.16	&	56.8/58.5	 \\ \hline
\rbm & {\color{red}3.3/3.4}  	& 	{\color{red}72.1/71.6} 	 	& {\color{red}6.2/5.7}	&	{\color{red}58.3/59.3} 	\\ \hline
\onek & 3.7/3.8	&      70.1/69.7		 	&6.8/6.2	&	57.0/58.3	\\ \hline
\ak &3.6/3.8   	& 70.3/70.0 	 & 6.7/6.0 	&57.1/58.5  \\ \hline
\mk & 3.7/3.8 	& 70.3/69.9 	 &6.7/6.1 	&57.1/58.4  \\ \hline
\ck & {\color{blue}3.5/3.6} & {\color{blue}71.0/70.4} &{\color{blue}6.5/5.7} 	&{\color{blue}57.3/58.8}   \\ \hline
\end{tabular}

\end{table}

\begin{table}
\centering
\caption{Performance of \rbm acoustic models}
\label{tDNN}
\begin{tabular}{|c|c|c|c|c|}\hline
& \multicolumn{2}{|c|}{Bengali} & \multicolumn{2}{|c|}{Cantonese}  \\ \hline
(h, L)   & \textsf{perp}  & \textsf{acc} (\%) & \textsf{perp}  & \textsf{acc} (\%) \\ \hline \hline
$(1, 500)$  &   3.9/3.9 	&	69.2/69.3  	 	& 7.1/6.4	&	55.8/57.4	 \\ \hline
$(2, 500)$  &   3.5/3.6 	&	70.9/70.7  	 	& 6.6/6.1	&	57.3/58.4	 \\ \hline
$(3, 500)$  &   3.5/3.5 	&	71.2/70.9  	 	& 6.4/5.9	&	57.7/58.6	 \\ \hline
$(4, 500)$  &   3.4/3.5 	&	71.2/70.8  	 	& 6.4/5.9	&	57.5/58.7	 \\ \hline
$(1, 1000)$  &   3.7/3.7 	&	70.1/70.1  	 	& 6.8/6.2	&	56.4/58.0	 \\ \hline
$(2, 1000)$  &   3.4/3.4 	&	71.6/71.4  	 	& 6.3/5.8	&	58.2/59.0	 \\ \hline
$(3, 1000)$  &   3.4/3.5 	&	71.7/71.3  	 	& 6.3/5.7	&	58.0/59.2	 \\ \hline
$(4, 1000)$  &   3.3/3.5 	&	71.8/71.4  	 	& 6.6/5.8	&	57.1/58.6	 \\ \hline
$(1, 2000)$  &   3.6/3.7 	&	70.5/70.3  	 	& 6.7/6.1	&	56.9/58.1	 \\ \hline
$(2, 2000)$  &   3.4/3.4 	&	71.8/71.4  	 	& 6.2/5.7	&	58.3/59.3	 \\ \hline
$(3, 2000)$  &   3.4/3.5 	&	71.5/71.2  	 	& 6.2/5.6	&	57.8/59.1	 \\ \hline
$(4, 2000)$  &   3.3/3.4 	&	72.1/71.6  	 	& 6.4/5.8	&	57.8/59.1	 \\ \hline
\end{tabular}

\end{table}

\begin{table}[h]
\small
\centering
\caption{Performance of single Laplacian kernel}
\label{tKernel}
\vspace{0.5em}
\begin{tabular}{|c|c|c|c|c|}\hline
& \multicolumn{2}{|c|}{Bengali} & \multicolumn{2}{|c|}{Cantonese}  \\ \hline
Dim   & \textsf{perp}  & \textsf{acc} (\%) & \textsf{perp} & \textsf{acc} (\%) \\ \hline \hline
\textsf{2k} &4.4/4.4 	&	66.5/66.8  	 	& 8.5/7.4 & 52.7/54.8	 \\ \hline
\textsf{5k} &4.1/4.2  	& 	67.8/67.8 	 	& 7.8/7.0 & 53.9/56.0		\\ \hline
\textsf{10k} &4.0/4.1	&      68.4/68.3		 	&7.5/6.7  &  54.9/56.6 	\\ \hline
\textsf{25k} &3.8/3.9   	& 69.2/69.0 	 &7.1/6.4&55.9/57.3   \\ \hline
\textsf{50k} & 3.8/3.9 	& 69.7/69.4 	 &6.9/6.2&56.5/57.9  \\ \hline
\textsf{100k} & 3.7/3.8 & 70.0/69.6 		&6.8/6.2&56.8/58.2   \\ \hline
\textsf{200k} & 3.7/3.8 &70.1/69.7 	&6.8/6.2&57.0/58.3  \\ \hline
\end{tabular}
\end{table}

\begin{table}
\small
\centering
\caption{Best Token Error Rates on Test Set ($\%$)}
\label{tTERComp1}
\begin{tabular}{|c|c|c|}\hline
Model  & Bengali & Cantonese \\ \hline\hline
\ibm & 70.4  	&  	67.3	 \\ \hline
\rbm & {\color{red}69.5} 	& 	66.3 \\ \hline
\onek &  70.0	& 	{\color{red}65.7}	\\ \hline
\ak &  73 &  68.8  \\ \hline
\mk & 72.8  & 69.1 \\ \hline
\ck & 71.2 &  68.1  \\ \hline \hline
\end{tabular}
\end{table}

\subsubsection{Results on Perplexity and Accuracy}

Table~\ref{tPerpAcc} concisely contrasts the best  perplexity and accuracy attained by various systems: \ibm (IBM's DNN), \rbm (RBM-trained DNN), \onek (single kernel based model), \ak (additive combination of two kernels), \mk (multiplicative combination of two kernels) and \ck (composite of two kernels). We report the metrics on both the held-out and the test datasets (the numbers are separated by a /). In general, the metrics are consistent across both datasets and \textsf{perp} correlates with \textsc{acc} reasonably well. 

On Bengali, across all systems, the \rbm attains the best perplexity (red colored numbers in the table), outperforming \ibm and suggesting that unsupervised pre-training is advantageous. The best performing kernel model is \ck, trailing slightly behind \rbm and \ibm.

Similarly, on Cantonese, \rbm performs the best, followed by \ck, both outperforming \ibm. As an illustrate example, we show in Table~\ref{tDNN} the performance of \rbm on Bengali, under different types of architectures ($h$ is the number of hidden layers and $L$ the number of hidden units).  Meanwhile, in Table~\ref{tKernel}, we show the performance of single Laplacian kernel acoustic model with different number of random features.

Contrasting these two tables, it is interesting to observe that kernel models use far more parameters than DNNs to achieve similar perplexity and accuracy.  For instance, for a \rbm with $(h=1, L= 500)$ with a perplexity of $3.9$, the number of parameters is $360 \times 500 + 500 \times 1000 = 0.68$ million. This is  a fraction of a comparable kernel model with Dim=\textsf{10k}$ \times 1000 = 10$ million parameters. In some way, this ratio provides an intuitive measure of  the price being convenient, i.e., using random features in kernel models instead of adapting features to the data as in DNNs.

\subsubsection{Results on Token Error Rates}

Table~\ref{tTERComp1} reports the performance of various models measured in \textsf{TER}, another important and more relevant metric to speech recognition errors.

Note that the RBM-trained DNN (\rbm ) performs the best on Bengali, but our best kernel model performs the best on Cantonese.  Both perform better than IBM's DNN. On Cantonese, the improvement of our kernel model over \ibm is noticeably large ($1.6\%$ reduction in absolute).

Table~\ref{tTERComp2} highlights several interesting comparison between \rbm and kernel models. Concretely, it seems that DNNs need to  be  big enough in order to reach the proximity of its best \textsc{TER}. On the other end, the kernel models' performance plateaus rather quickly.   This is the opposite to what we have observed when we compare two methods using perplexity and accuracy.

One possible explanation is that for different models, the relationship between perplexity and \textsc{TER} are different. This is certainly plausible, given \textsc{TER} is highly complex to compute and two different models might explore parameter spaces very differently.

Another possible explanation is that these two different models learn different representations that bias either toward perplexity or toward \textsc{TER}.  Table~\ref{tTERComb} suggests that this might indeed be true: as we combine two different models, we see handsome gains in performance over each individual one's.

\begin{table}
\small
\centering
\caption{Detailed Comparison on \textsc{TER} for Bengali }
\label{tTERComp2}
\begin{tabular}{|c|c|c|}\hline
Model  & Arch. & \textsf{TER} ($\%$) \\ \hline\hline
\rbm & $h=2, L=1000$	& 	73.1 \\ \hline
\rbm & $h=3, L=1000$	& 	72.7 \\ \hline
\rbm & $h=2, L=2000$ & 72.4\\ \hline
\rbm & $h=3, L=2000$ & 72.2\\ \hline
\rbm & $h=4, L=1000$ & 69.8\\ \hline
\rbm & $h=4, L=2000$ & 69.5\\ \hline
\onek & Dim = 25k & 73.1\\ \hline
\onek & Dim = 50k & 70.2\\ \hline
\onek & Dim = 100k & 70.0\\ \hline
\onek & Dim = 200k & 70.0\\ \hline
\end{tabular}
\end{table}

\begin{table}
\small
\centering
\caption{Token Error Rates ($\%$) for Combined Models }
\label{tTERComb}
\begin{tabular}{|c|c|c|}\hline
Model  & Bengali & Cantonese \\ \hline
\textsc{best single system} & 69.5 & 65.7 \\ \hline
\rbm $(h=3, L=2000)$ + \onek & 69.7 & 65.3 \\ \hline
\rbm $(h=4, L=1000)$ + \onek & 69.2 & 64.9 \\ \hline
\rbm $(h=4, L=2000)$ + \onek & 69.1 & 64.9 \\ \hline
\end{tabular}
\end{table}

\subsubsection{DNN and kernels learn complementary representations}

Inspired by what we have observed in the previous section, we set out to analyze in what way the representations learnt by two different models might be complementary. We have obtained preliminary results.

We took a learned DNN (we used the best perform one in terms of \textsf{TER}) and computed its pre-activation to the output layer, which is a linear transformation of the last hidden layer's outputs. For the best performing single-kernel model,  we computed the pre-activation similarly. Note that since they both predict the same set of labels, the pre-activations from either model have the same dimensionality.

We perform PCA on them independently and then visualize in 2D. Fig.~\ref{fRep} displays the two scatter plots where each has 1000 points, representing the means of the learned representations for data points in each class. To visualize easily, we color each point not by its phoneme state labels. Instead, we collapse them into phone labels (which are considerably few, generally around 40 - 60). 

An initial examination seems to suggest that kernel models' representations tend to form clumps for data from the same class. In the figure, the most obvious observation is the cluster in the blue color. On the other end, those blue color scattered data points do not seem to form a large and tight cluster 
under the representations learned by the DNNs -- they seem to be more spread out.

The clumps seem to be indicative of the Gaussian kernels we have used. However, how important they are and in what way, the more flourish patterns by DNNs' representations are more advantageous require more careful and detailed analysis. We hope our work has provided enough incentives and tools for that pursuit. 

\begin{figure}
\centering
\includegraphics[width=0.49\columnwidth]{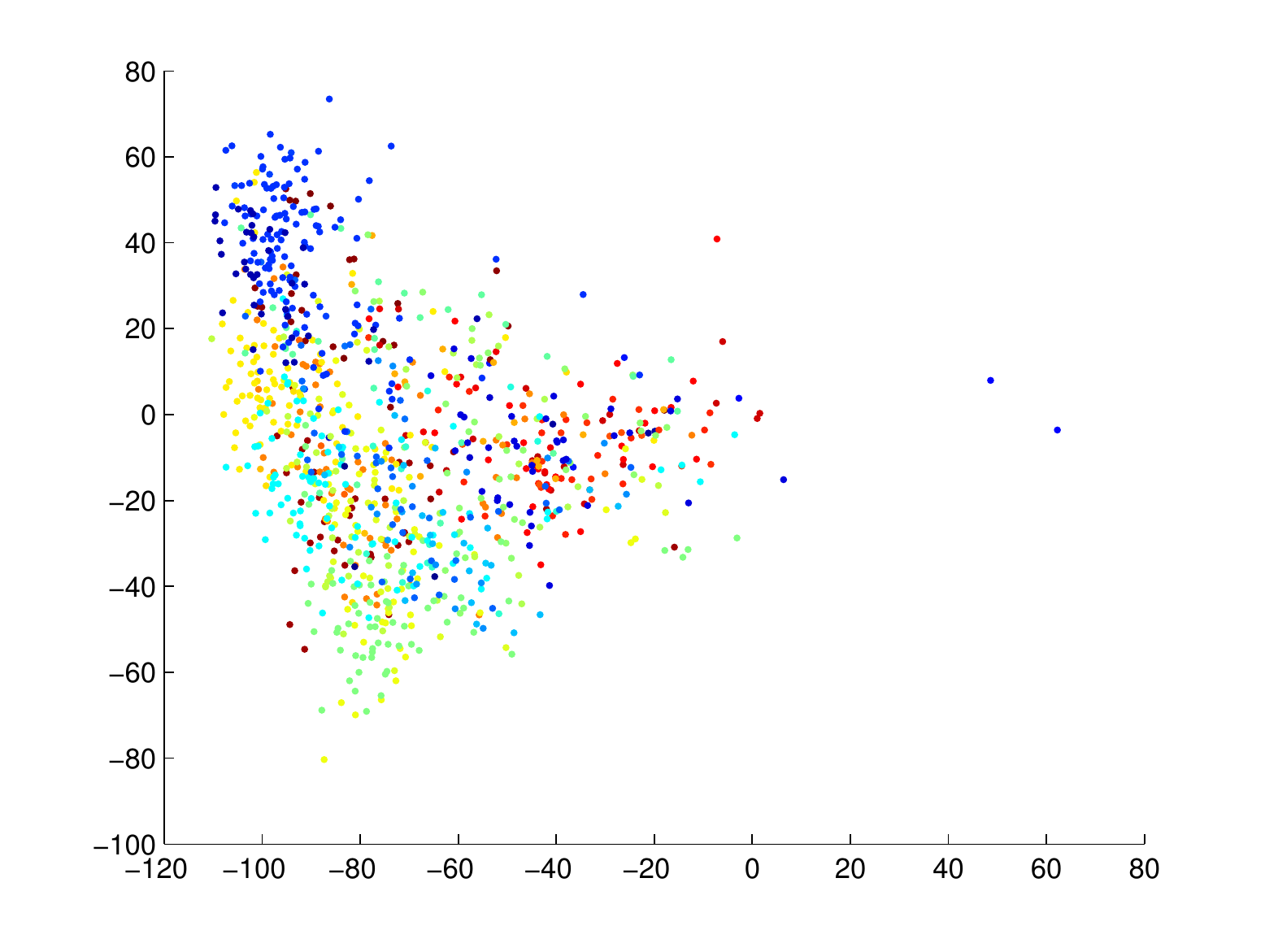}
\includegraphics[width=0.49\columnwidth]{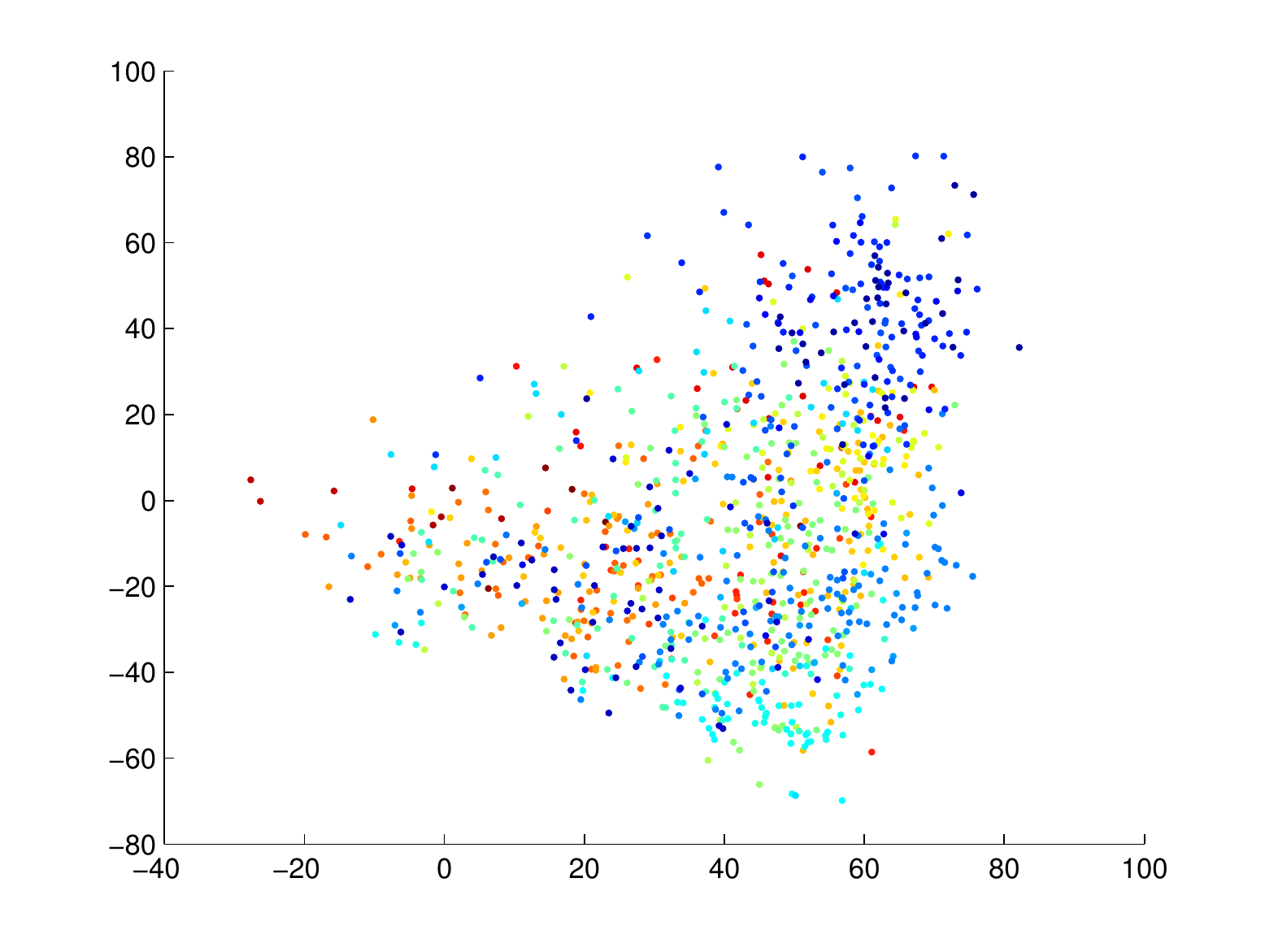}
\caption{Contrast the learned representations by kernel models (left plot) and DNNs (right plot)}
\label{fRep}
\end{figure}

\end{document}